\AtBeginDocument{%
  \providecommand\BibTeX{{%
    \normalfont B\kern-0.5em{\scshape i\kern-0.25em b}\kern-0.8em\TeX}}}

\PassOptionsToPackage{table}{xcolor}
\PassOptionsToPackage{dvipsnames}{xcolor}
\documentclass[sigconf]{acmart}

\newcommand\vldbdoi{XX.XX/XXX.XX}
\newcommand\vldbpages{XXX-XXX}
\newcommand\vldbvolume{14}
\newcommand\vldbissue{1}
\newcommand\vldbyear{2025}
\newcommand\vldbauthors{\authors}
\newcommand\vldbtitle{\shorttitle} 
\newcommand\vldbavailabilityurl{}
\newcommand\vldbpagestyle{plain} 

\usepackage{bm}
\usepackage{url}
\usepackage{xcolor}
\usepackage{bbm}
\usepackage{graphicx}
\usepackage[dvipsnames]{xcolor}
\usepackage{makecell}
\usepackage[table]{xcolor}
\usepackage[noend]{algorithmic}
\usepackage{caption}
\usepackage{amsmath}

\usepackage{diagbox}
\usepackage{multirow}
\usepackage{enumitem}
\usepackage{subcaption}
\usepackage{relsize}
\usepackage{tabularx}
\usepackage{color}
\usepackage{longtable}
\usepackage{supertabular}

\usepackage{mdframed}
\newmdenv[linecolor=black, backgroundcolor=gray!10]{boxedalgorithm}
\usepackage[ruled,vlined]{algorithm2e}
\usepackage{fancyvrb}
\usepackage{xcolor}
\usepackage{xspace}
\usepackage{color} 
\usepackage{layouts}
\usepackage[utf8]{inputenc}
\usepackage{eqparbox}
\usepackage{etoolbox} 
\usepackage{pgfplots}
\usepackage{subcaption} 
\usepackage{marginnote}
\usepackage{caption}
\usepackage{listings}
\usepackage{tabularx}
\usepackage{amsmath}
\usepackage{xspace}
\usepackage{mathtools}
\usepackage{amsthm}
\usepackage{amssymb} 
\usepackage{pifont}
\usepackage{caption}
\usepackage{layouts}
\usepackage{tikz}
\usetikzlibrary{arrows.meta, positioning, shapes.geometric}
\definecolor{darkblue}{rgb}{0.0, 0.0, 0.55}
\definecolor{maroon}{rgb}{0.5, 0.0, 0.0}
\usepackage[normalem]{ulem}

\newcounter{nalg}[section] 
\renewcommand{\thenalg}{\arabic{nalg}} 
\DeclareCaptionLabelFormat{algocaption}{Algorithm \thenalg} 

\lstnewenvironment{algo}[1][] 
{   
    \refstepcounter{nalg} 
    \captionsetup{labelformat=algocaption,labelsep=colon} 
    \lstset{ 
        mathescape=true,
        frame=tB,
        numbers=left, 
        numberstyle=\tiny,
        basicstyle=\scriptsize, 
        keywordstyle=\color{black}\bfseries\em,
        keywords={,input, output, return, datatype, function, in, if, else, foreach, while, begin, end, } 
        numbers=left,
        xleftmargin=.04\textwidth,
        #1 
    }
}
{}

\tikzset{%
  every neuron/.style={
    circle,
    draw,
    minimum size=1cm
  },
  neuron missing/.style={
    draw=none, 
    scale=4,
    text height=0.333cm,
    execute at begin node=\color{black}$\vdots$
  },
}

\newcommand*\circled[1]{\tikz[baseline=(char.base)]{
            \node[shape=circle,draw,inner sep=0.2pt] (char) {#1};}}

\author{Peizhi Wu}
\affiliation{%
 \institution{University of Pennsylvania}
 \postcode{19104}
}
\email{pagewu@cis.upenn.edu}
\author{Haoshu Xu}
\affiliation{%
 \institution{University of Pennsylvania}
 \postcode{19104}
}
\email{haoshuxu@sas.upenn.edu}
\author{Ryan Marcus}
\affiliation{%
 \institution{University of Pennsylvania}
 \postcode{19104}
}
\email{rcmarcus@cis.upenn.edu}
\author{Zachary G. Ives}
\affiliation{%
 \institution{University of Pennsylvania}
 \postcode{19104}
}
\email{zives@cis.upenn.edu}

\definecolor{mycolor}{RGB}{255,0,0}

\newcommand{\eat}[1]{}

\newcommand{\neucdf}{\mbox{$\textsc{NeuroCDF}$}\xspace}
\newcommand{\name}{\mbox{$\textsc{SeConCDF}$}\xspace}
\newcommand{\pts}{\mbox{$\textsc{PtsHist}$}\xspace}
\newcommand{\quadh}{\mbox{$\textsc{QuadHist}$}\xspace}
\newcommand{\quick}{\mbox{$\textsc{Quicksel}$}\xspace}
\newcommand{\samp}{\mbox{$\textsf{Sampling}$}\xspace}
\newcommand{\pg}{\mbox{$\textsf{PostgreSQL}$}\xspace}

\newenvironment{sproof}{%
  \renewcommand{\proofname}{Proof Sketch}\proof}{\endproof}

\newcommand{\set}[1]{\mbox{${\bm{\mathcal{{#1}}}}$\xspace}}
\newcommand{\vcd}[1]{\mbox{\xspace\texttt{VC-dim} ({#1})\xspace}}

\newtheorem{assumption}{Assumption}
\numberwithin{assumption}{section}
\newtheorem{example}{Example}
\numberwithin{example}{section}
\newtheorem{lemma}{Lemma}
\numberwithin{lemma}{section}
\newtheorem{proposition}{Proposition}
\numberwithin{proposition}{section}

\newcommand{\abs}[1]{\left|#1\right|} 

\newcommand{\cbr}[1]{\left\{#1\right\}}     
\newcommand{\rbr}[1]{\left(#1\right)}       
\newcommand{\sbr}[1]{\left[#1\right]}       
\newcommand{\tv}[2]{\mathop{TV}\rbr{#1,#2}}

\newcommand{\er}{\mathrm{er}}

\def\cT{\mathcal{T}}
\def\cX{\mathcal{X}}

\def\hf{\hat{f}}
\def\hS{\hat{S}}

\def\EE{\mathop{\mathbb{E}}}
\def\RR{\mathbb{R}}
\newcommand{\1}{\mathbbm{1}} 

\DeclareMathOperator{\supp}{supp}

\newcommand{\sign}{\mathrm{sgn}}

\newcommand{\fatt}[1]{\mbox{$\textrm{fat}\rbr{#1;\gamma}$}\xspace}
\newcommand{\ceil}[1]{\lceil {#1} \rceil}
\newcommand{\flr}[1]{\lfloor {#1} \rfloor}

\newcommand{\la}{\langle}
\newcommand{\ra}{\rangle}

\makeatletter
\newtheoremstyle{mytheorem}
  {3pt}
  {3pt}
  {\itshape}
  {}
  {\itshape\bfseries}
  {.}
  {.5em}
  {\thmname{#1}\thmnumber{\@ifnotempty{#1}{ }#2}%
   \thmnote{ {\the\thm@notefont(#3)}}}
\makeatother
\theoremstyle{mytheorem}
\newtheorem{theorem}{Theorem}
\numberwithin{theorem}{section}

\begin{document}

 

\title{A Practical Theory of Generalization in Selectivity Learning}

\begin{abstract}
\emph{Query-driven machine learning models} have emerged as a promising estimation technique for query selectivities. 
Yet, surprisingly little is known about the efficacy of these techniques from a theoretical perspective, as there exist substantial gaps between practical solutions and state-of-the-art (SOTA) theory based on the Probably Approximately Correct (PAC) learning framework. In this paper, we aim to bridge the gaps between theory and practice. First, we demonstrate that selectivity predictors induced by \emph{signed measures} are learnable, which relaxes the reliance on \emph{probability measures} in SOTA theory. More importantly, beyond the PAC learning framework (which only allows us to characterize how the model behaves when both training and test workloads are drawn from the \emph{same} distribution), we establish, under mild assumptions, that selectivity predictors from this class exhibit favorable \emph{out-of-distribution} (OOD) generalization error bounds. 

These theoretical advances provide us with a better understanding of both the in-distribution and OOD generalization capabilities of query-driven selectivity learning, and facilitate the design of two general strategies to improve OOD generalization for existing query-driven selectivity models. We empirically verify that our techniques help query-driven selectivity models generalize significantly better to OOD queries both in terms of prediction accuracy and query latency performance, while maintaining their superior in-distribution generalization performance.

\end{abstract}

\maketitle

\eat{
\pagestyle{\vldbpagestyle}
\begingroup\small\noindent\raggedright\textbf{PVLDB Reference Format:}\\
\vldbauthors. \vldbtitle. PVLDB, \vldbvolume(\vldbissue): \vldbpages, \vldbyear.\\
\href{https://doi.org/\vldbdoi}{doi:\vldbdoi}
\endgroup
\vspace{-0.3em}
\begingroup
\renewcommand\thefootnote{}\footnote{\noindent
This work is licensed under the Creative Commons BY-NC-ND 4.0 International License. Visit \url{https://creativecommons.org/licenses/by-nc-nd/4.0/} to view a copy of this license. For any use beyond those covered by this license, obtain permission by emailing \href{mailto:info@vldb.org}{info@vldb.org}. Copyright is held by the owner/author(s). Publication rights licensed to the VLDB Endowment. \\
\raggedright Proceedings of the VLDB Endowment, Vol. \vldbvolume, No. \vldbissue\ %
ISSN 2150-8097. \\
\href{https://doi.org/\vldbdoi}{doi:\vldbdoi} \\
}\addtocounter{footnote}{-1}\endgroup

\ifdefempty{\vldbavailabilityurl}{}{
\vspace{.3cm}
\begingroup\small\noindent\raggedright\textbf{PVLDB Artifact Availability:}\\
The source code, data, and/or other artifacts have been made available at \url{\vldbavailabilityurl}.
\endgroup
}
}

\section{Introduction}

We study the learning of selectivity functions for selection queries in database management systems (DBMSes). 
As the key to effective query optimization, selectivity estimation has continued to be one of the most important problems in DBMSes since the 1980s~\cite{selinger1979access,lynch1988selectivity}. The earliest approach was to collect basic statistics (such as histograms) for selectivity estimation, and then to make uniformity (within a bucket) and independence (among columns) assumptions. Although widely adopted in real DBMSes due to its simplicity, this approach is prone to large estimation errors~\cite{ioannidis1991propagation,leis2015good}.

More recently, selectivity estimation has been formulated as a machine learning (ML) problem, where the system learns from observed samples (data or queries) to make selectivity predictions for incoming queries. Proposals for learning-based selectivity estimation can be broadly categorized into \emph{data-driven} and \emph{query-driven} models (with a few exceptions in the form of hybrid models). Data-driven techniques~\cite{yang2020neurocard, wu2023factorjoin, kim2024asm, heddes2024convolution, meng2023selectivity,zhang2024duet, deeds2023safebound, deepdb} build models of the data distribution by scanning the underlying data. Conversely, query-driven techniques either learn a regression model from query features to selectivity~\cite{kipf2019learned,chen1994adaptive}, or model the data distribution from a set of observed queries and their selectivities~\cite{aboulnaga1999self}. 

In this paper, we focus on query-driven models~\cite{kipf2019learned, dutt2019selectivity, reiner2023sample, wang2023speeding, li2023alece, park2020quicksel,hu2022selectivity} as they enjoy a smaller model size, faster training, and possibly faster inference (for example, regression models~\cite{kipf2019learned, dutt2019selectivity, li2023alece}) compared to data-driven models. In addition, they can also achieve much better performance than traditional histograms~\cite{kipf2019learned}. 

\smallskip

\vspace{-0.2em}
\noindent \textbf{Importance of theoretical understanding of generalization. }
In machine learning, \emph{generalization} (``a central goal in pattern recognition~\cite{bishop2006pattern}'') refers to a model's ability to perform well on new, unseen data that was not part of the training set. With respect to query-driven selectivity learning, the large variability in queries seen in practice means that any training workload can represent only a \emph{tiny} subset of all possible queries. Therefore, it is crucial to accurately characterize the generalization ability of selectivity models, specifically how they perform on queries that were \emph{not} seen during training. This understanding is essential to ensure reliable predictions in real-world applications.
Yet, surprisingly, there is limited theoretical analysis of the generalizability of query-driven models. An initial and promising step towards such understanding~\cite{hu2022selectivity} proves that selectivity functions are learnable using the Probably Approximately Correct (PAC) learning framework~\cite{kearns1994introduction}. However, significant gaps remain in our understanding.

\smallskip

\vspace{-0.2em}
\noindent \textbf{Limitations of prior results. } 
The current SOTA result~\cite{hu2022selectivity} assumes that \emph{every selectivity predictor in the hypothesis class is induced by a probability measure}. Consequently, \emph{learnability} (in-distribution generalization, to be formally introduced in \S~\ref{subsec: PAC}) results can be applied \emph{only} to a small fraction of existing query-driven models (\textit{e.g.,} those that build histograms from queries~\cite{aboulnaga1999self}). Indeed, as we will see later in the paper, predictors from regression-based query-driven models, which achieve impressive empirical performance, are not induced by a probability measure. Therefore, existing learnability results~\cite{hu2022selectivity} cannot be applied to these practical approaches. Given this gap between theory and practice, a natural question arises:

\vspace{-0.2em}
\medskip
\begin{minipage}{0.44\textwidth}
\paragraph{\textbf{Question 1:}} \textit{Is it feasible to reduce the reliance on probability measures, thereby broadening our theoretical understanding of selectivity learning models?}
\end{minipage}
\medskip

\noindent Another challenge in applying the theoretical results to practical scenarios is that PAC learning, as a framework, \emph{only} allows us to quantify the \emph{in-distribution} generalization error, where both training and test queries are drawn from the \emph{same} distribution.  This means that previous theory~\cite{hu2022selectivity} based on PAC learning is \emph{not} able to characterize generalization error for OOD scenarios. Nevertheless, in the real world, query workloads may shift constantly~\cite{negi2023robust, wu2024modeling}. This raises another, perhaps more challenging, question:

\medskip
\begin{minipage}{0.44\textwidth}
\paragraph{\textbf{Question 2:}} \textit{Given mild assumptions, is it feasible to quantify \textbf{OOD generalization error} in selectivity learning, thereby enhancing the practical relevance of theoretical results?}
\end{minipage}
\medskip


\noindent Our first goal in this paper is to answer these two questions theoretically. Thereafter, based on the new generalization results, we design new learning paradigms/frameworks for improving selectivity estimation in practice, which leverage the theoretical results to provide formal guarantees.



\smallskip

\noindent \textbf{A sketch of our results. } The paper delivers two \emph{positive and encouraging} theoretical results toward answering the two questions:  

\begin{itemize} [leftmargin=*]
    \item Addressing Question 1, we introduce a new theoretical result of learnability (\textit{i.e.,} in-distribution generalization) that applies to selectivity functions/models whose predictions are induced by a \emph{signed measure}, removing the positivity and sum-to-unity constraints that are required by prior work.

    \item More interestingly, under mild assumptions, we establish \emph{non-trivial} OOD generalization error bound for selectivity predictors that are induced by a signed measure.  The new result, \emph{beyond the PAC learning framework}, quantifies the generalization error when training and testing workloads do not follow the same distribution, hence answering Question 2. 
    For a taste of our theory, our main theorem (Theorem \ref{thm: ood}) is simplified below.

    \medskip
    \noindent\fbox{\begin{minipage}{0.44\textwidth}
    \paragraph{\textbf{Simplified Theorem \ref{thm: ood}}}\textit{For any selectivity estimator $\hat{S}$ that is induced by a signed measure, if $\hat{S}$ is trained under distribution $Q$ with in-distribution generalization error $\er_Q(\hat{S})$ upper bounded by $\epsilon$ with probability at least $1-\delta$,
    then under a different testing distribution $P$, the out-of-distribution generalization error $\er_P(\hat{S})$ satisfies}
    \begin{align*}
        \er_P(\hat{S}) \leq O(\sqrt{\epsilon}) 
    \end{align*} 
    \textit{with probability at least $1-\delta$, under mild assumptions on distribution $P$ and $Q$ (see Theorem \ref{thm: ood} for details). }
    \end{minipage}}
    \medskip
\end{itemize}

\noindent A key implication of our result is that, for any class of selectivity predictors that is induced by signed measures, both our in-distribution and OOD generalization results apply \emph{immediately}.

\medskip
\noindent \textbf{Improvement strategies inspired by our theory}. From this aspect of our theory, we propose novel and practical methodologies for \emph{improving existing query-driven selectivity learning models.}
\begin{itemize} [leftmargin=*]
    \item We propose a new modeling paradigm for query-driven selectivity learning, \neucdf, which models the underlying cumulative distribution functions (CDFs) using a neural network. \neucdf is proved to be induced by signed measures, and thus enjoys the theoretical guarantees of our theory, and enjoys the superior empirical performance of deep learning. Although challenging to optimize with relative error metrics like Qerror, \neucdf \emph{provably} offers better generalization performance for OOD queries, compared to the common paradigm for selectivity estimation that targets the query selectivity directly.

    \item Inspired by the lessons learned from our theory and \neucdf, we propose a general training methodology for enhancing existing query-driven selectivity models. \name incorporates the idea of CDF modeling of \neucdf into query-driven models by enforcing model \textbf{Se}lf-\textbf{Con}sistency with the learned \textbf{C}umulative \textbf{D}istribution \textbf{F}unctions. However, unlike \neucdf, \name keeps the original loss functions (Qerror or RMSE) of existing query-driven models, which allows for good in-distribution generalization with either relative or absolute loss functions. Moreover, 
    the CDF self-consistency training of \name significantly enhances model OOD generalization ability.
\end{itemize}
\noindent \textbf{Takeaways from the experiments. } Note that the proposed improvement strategies are \emph{orthogonal} to selectivity model architectures, making them applicable to various existing models. Our primary goal is \emph{not} to outperform current SOTA query-driven selectivity learning models, but to validate the practicality of our theory by designing algorithms that improve the OOD generalization capabilities of existing models \emph{with theoretical guarantees.} Thus, we focus our experimental evaluation on aspects in which our strategies are expected to provide improvements. Indeed, this focused approach has yielded clear, compelling results: across both single- and multi-table datasets, our strategies can \emph{significantly} improve the OOD generalization of existing selectivity learning models, in terms of both estimation accuracy (\textit{i.e.,} smaller Qerror and RMSE) and query running time performance (\textit{i.e.,} lower query latency).

\smallskip
\noindent \textbf{Organization.} This paper is organized as follows: Section \S~\ref{section.prior} reviews prior work on query-driven selectivity learning.  \S~\ref{section.preliminaries} outlines definitions and the problem setup. In \S~\ref{section.theory}, we introduce our new theory, followed by two improvement strategies in \S~\ref{section.neucdf} and \S~\ref{section.training}. Our algorithms are evaluated in \S~\ref{section.eval}, and we conclude in  \S~\ref{section.conclusion}.

\section{Prior Work~\label{section.prior}} 

Selectivity estimation dates back to the beginning of query processing~\cite{selinger1979access}, where rather than computing intermediate results and then finding query plans~\cite{wong1976decomposition}, System-R instead used histograms and independence assumptions. Such techniques were refined to use queries themselves to compute histograms~\cite{aboulnaga1999self, bruno2001stholes, lim2003sash},  query expression statistics~\cite{bruno2002exploiting} and adjustments to correlated predicates~\cite{markl2003leo}. More recent \emph{learned data-driven methods}~\cite{deepdb, yang2020neurocard} do offline computation over samples of existing database instances to build models of data distributions in the presence of skew and correlations. \emph{Learned query-driven database systems} can learn or improve an ML model for a variety of database components, by using the execution log of a query workload~\cite{anagnostopoulos2017query,anagnostopoulos2015learning,anagnostopoulos2015learningidcm,xiu2024parqo,li2023alece}. More recently, there is active work on workload-aware cardinality predictors~\cite{wu2018towards,kipf2019learned,wu2021unified}.
In this paper, we consider several families of selectivity estimation techniques.

\noindent
\textbf{Parametric Functions~\cite{chen1994adaptive}.} The early approach fits a parametric function (\textit{e.g.,} linear and polynomial) to observed queries. These functions take a query as input and produce a selectivity estimate. However, the performance of parametric functions is not as good as more recent approaches due to the limited model capacity.

\noindent \textbf{Histograms~\cite{aboulnaga1999self, markl2007consistent}.}  
Histogram-based models, widely studied in database literature, build histograms from query workloads by adjusting bucket frequencies to correct prior errors or by aligning with a maximum entropy distribution with observed queries. They assume uniformity within buckets and independence across columns (or features), which could lead to large estimation errors.

\noindent \textbf{LEO~\cite{stillger2001leo}.} Intuitively, LEO can be seen as a combination of parametric functions and histograms --- it learns the adjustment factors from observed queries to correct incorrect statistics such as histograms. Specifically, LEO collects a set of previous ratios $r = \frac{act\_sel}{stat\_set}$ of actual selectivity ($act\_set$) and statistics estimate ($stat\_set$) from past queries. To estimate an incoming query, LEO uses the ratios to adjust the statistics estimate by multiplying it by a chosen adjustment ratio $r$. For example, consider a query asks for the range $\{x<1\}$ and the selectivity estimate of the histogram for the query is $\hat{Hist}(x<1)$. LEO produces the adjusted estimate by $\text{adjusted\_sel} = \hat{Hist}(x<1) * r(x<1)$, where $r(x<1)$ is the collected adjustment factor at $x=1$. If there is no adjustment factor for $x=1$, LEO computes the factor by linear interpolation.

\noindent \textbf{Deep Learning Models~\cite{kipf2019learned, dutt2019selectivity}.} More recently, deep learning models have been proposed to learn the mapping from a query to its selectivity prediction.
Deep learning models function as regression models in a way that is similar to parametric functions but has a larger model capacity and much better performance.

\section{\mbox{Preliminaries and Problem Setup}~\label{section.preliminaries}}

In this section, we start by defining key concepts for selectivity estimation in \S~\ref{subsec: select}. We then introduce measure theory in \S~\ref{subsec: sgn_m} due to its connection with selectivity functions and its importance in shaping our theory. Next, we frame selectivity estimation as a learning problem in \S~\ref{subsec: learning}, discuss the PAC learning framework in \S~\ref{subsec: PAC}, and review existing theoretical results in \S~\ref{section.review}. The section concludes with an analysis of the probability measure assumption in \S~\ref{section.gap}, motivating the goals of this paper.

\subsection{Selectivity Functions of Range Queries} \label{subsec: select}

\textbf{Range Space.} Consider a $d$-dimensional dataset $D$. A range space is defined as $\Sigma = (\mathcal{X}, \bm{\mathcal{R}})$. $\mathcal{X}$ is a set of objects (\textit{e.g.,} tuples or data points in $D$). $\set{R}$ is a collection of ranges $R$, which is \emph{a subset of $\mathcal{X}$}. For instance, $\set{R}$ can be a set of all $d$-dimensional hyper-rectangles. 

\noindent \textbf{Range Queries.} A range query $q$ is defined as a query that retrieves tuples within the range $R_q$. Thus range query $q$ and its querying range $R_q$ are interchangeable. We focus on range selection queries, corresponding to \emph{$d$-dimensional hyper-rectangles.} Join queries can be viewed as range selection queries over the join result.

\noindent \textbf{Selectivity (Cardinality) Functions.} For a dataset $D$, let $P_D$ be the data probability distribution over $D$, we define the selectivity functions as $S_D(R) = P_{x \sim P_D} (x \in R)$, or equivalently,
\begin{equation} \label{eq.selfunc2}
    S_D(R) = \sum_{x \in R}P_D(x)
\end{equation}
Another term is cardinality (the output size of a range query). The relationship between cardinality $C_D(R)$ and selectivity can be written as $C_D(R) = S_D(R) \cdot |T|$ where $|T|$ is the size of table $T$. 

\vspace{-0.5em}
\subsection{Measure Theory} \label{subsec: sgn_m}
\noindent  \textcolor{black}{\textbf{Basic Concepts.} We first formally introduce fundamental notations from measure theory that will be used to shape our theorems}.
A $\sigma$-algebra $\set{M}$ of “measurable” sets is a non-empty collection of subsets of $\cX$ closed under complements and countable unions and intersections. For all practical applications, it holds that $\set{M} \supset \set{R}$.

A function $\mu: \set{M} \to \RR$ is a \emph{probability measure} on $(\cX,\set{M})$ if it satisfies 
\begin{enumerate}[leftmargin=*,label=\textbf{C\arabic*.},ref=\textbf{C\arabic*}]
    \item \label{cond.add} 
    Countable additivity: if $E_1, E_2,...$ is a countable family of disjoint sets in $\set{M}$, then $\mu\rbr{\bigcup_{n=1}^{\infty} E_n} = \sum_{n=1}^{\infty} \mu(E_n)$.
    \item \label{cond.pos}
    Positivity: $\mu(E)\geq 0$ for any
    $E \in \set{M}$.
    \item \label{cond.one} Sum to unity: $\mu(\cX)=1$.
\end{enumerate} 
If $\mu$ only satisfies \ref{cond.add} and \ref{cond.pos}, it is called a \emph{measure}; if it only satisfies \ref{cond.add}, then it is a \emph{signed measure}. A signed measure is essentially the difference between two measures.

\textcolor{black}{We now define \textbf{\emph{induction}} for selectivity functions using measure theory.}
A selectivity estimate $\hat{S}: \set{R} \to \RR$ is said to be \textbf{\emph{induced}} by a (probability or signed) measure if there exists a  measure, denoted by $\mu_{\hat{S}}$, that satisfies  $\hat{S}(R)=\mu_{\hat{S}}(R) \text{ for all } R \in \set{R}$. \textcolor{black}{Intuitively, \ref{cond.add} implies the \textit{finite additivity} of selectivity functions: 
     $\hat{S}(R_1) = \hat{S}(R_2) + \hat{S}(R_3)$ 
  if $R_1 = R_2 \cup R_3$ and $ R_2 \cap R_3 = \emptyset$.  Moreover, \ref{cond.pos} requires that $\hat{S}$ only outputs positive values; \ref{cond.one} means that the values of $\hat{S}$ sum to 1 over the entire set of data points.}
\begin{figure}
  \centering
\begin{minipage}[c]{0.20\textwidth}
    \centering
    \tikzstyle{every node}=[font=\Large]

    \begin{tikzpicture} [scale=0.6, transform shape]
      \node[label=left:{$A$}] (A) at (0,0) {};
      \node[label=right:{$B$}] (B) at (4,0) {};
      \node[label=above:{$C$}] (C) at (2,3.7) {};

      \fill (A) circle (2pt);
      \fill (B) circle (2pt);
      \fill (C) circle (2pt);

      \node[draw, red, rotate around={0:(2,0)}, minimum width=4.5cm, minimum height=1.3cm, ellipse] at (2,0) {$R_{AB}$};

      \node[draw, blue, rotate around={60:(0,1.75)}, minimum width=4.5cm, minimum height=1.3cm, ellipse] at (-0.5,1) {$R_{AC}$};

      \node[draw, black, rotate around={-60:(0,1.75)}, minimum width=4.5cm, minimum height=1.3cm, ellipse] at (4.5,1) {$R_{BC}$};

      \node[draw, teal, minimum width=6cm, minimum height=5.7cm, ellipse] at (2,1.4) {$R_{ABC}$};
    \end{tikzpicture}
  \end{minipage}
  \hfill
  \begin{minipage}[c]{0.27\textwidth}
    \centering
    \begin{table}[H]
    \scalebox{0.85}{ 
        \begin{tabular}{|c|c|c|c|c|}
        \hline
         & $R_{ABC}$ &$R_{AB}$ &$R_{AC}$ &$R_{BC}$\\ \hline
         \rowcolor{LimeGreen!50} 
        $S_1$& 1 & 0.5 & 0.7 &  0.8\\ \hline
        \rowcolor{pink!50} 
         $S_2$& 1 & 0.3 & 0.3 & 0.3 \\ \hline
         \rowcolor{lightgray!50} 
        $S_3$ & 0.9 & 0.5 &  0.6 & 0.7 \\ \hline
        \rowcolor{lightgray!50} 
         $S_4$ & 1 & 0.4 & 0.5 & 1.1 \\ \hline
        \end{tabular}}
    \end{table}

    \vspace{-0.6cm} 
    
     \begin{table}[H]
     \centering
    \scalebox{0.85}{ 
        \begin{tabular}{|c|c|c|c|}
        \hline
         & $A$ &$B$ &$C$\\ \hline
         \rowcolor{LimeGreen!50} 
        $\mu_1$& 0.2  & 0.3 &  0.5\\ \hline
        \rowcolor{pink!50} 
         $\mu_2$& / & / & / \\ \hline
         \rowcolor{lightgray!50} 
        $\mu_3$ & 0.2 &  0.3 & 0.4 \\ \hline
        \rowcolor{lightgray!50} 
         $\mu_4$ & - 0.1 & 0.5 & 0.6 \\ \hline
        \end{tabular}}
    \end{table}
  \end{minipage}
  \caption{Left: three data points ($A,B,C$) and four ranges. Right Top: predictions from  selectivity functions ($S_1 \sim S_4$) for the four ranges. Right Bottom: corresponding measures ($\mu_1 \sim \mu_4$) that induce $S_1 \sim S_4$, including their measure values on $A,B,C$ (technically, the three sets $\{A\}, \{B\}, \{C\}$).} \label{fig.pred_example}
\end{figure}
\begin{example} \label{example.measure}

\textcolor{black}{
Figure~\ref{fig.pred_example} (left) gives an illustration of three data points ($A, B, C$) and four possible ranges ($R_{AB}, R_{AC}, R_{BC}, R_{ABC}$), with four range functions and their selectivity predictions on the right. We also show the measures that induce the selectivity predictions of each range function (with their outputs on the three data points) in the bottom-right table. First, using Eq.~\ref{eq.selfunc2} and basic linear algebra, one can see that $S_1$ is induced by a proper probability measure (\textit{e.g.,} $\mu_1(A)=0.2, \mu_1(B)=0.3, \mu_1(C)=0.5$). However, this does not hold for the other selectivity functions. Specifically, $S_2$ does not satisfy \ref{cond.add} as $S_2(R_{AB}) + S_2(R_{AC})+ S_2(R_{BC}) \neq 2 \cdot S_2(R_{ABC})$, which indicates that $S_2$ cannot be induced by a probability measure or a signed measure. Additionally,  $S_3, S_4$ can only be induced from a signed measure ---  $S_3$ violates \ref{cond.one} and $S_4$ violates \ref{cond.pos} of a probability measure.}
\end{example} 

\noindent  \textcolor{black}{\textbf{Advanced Concepts.} We introduce here two 
 concepts in measure theory that will appear only in our proofs; readers may skip this subsection at first.} Given a signed measure $\mu$, the \emph{total variation} of $\mu$, denoted by $\abs{\mu}$, is defined by $\abs{\mu}(E)= \sup \sum_{n=1}^{\infty} \abs{\mu(E_n)}$ where the supremum is taken over all partitions of $E$, that is, over all countable unions $E = \bigcup_{n=1}^{\infty} E_n$, where the sets $E_n$ are disjoint and belong to $\set{M}$. {\color{black} Intuitively,  $\abs{\mu}$ measures how much $\mu$ "varies" in its domain, and one can show that the total variation $\abs{\mu}$ itself is a measure that dominates $\mu$ ($\abs{\mu} \geq \mu$)}.

A signed measure $\mu$ is \emph{absolutely continuous} w.r.t. the Lebesgue measure $m$ if $\mu(E)=0$ whenever $E \in \set{M}$ and $m(E)=0$. {\color{black} Absolute continuity can be interpreted as the \emph{smoothness} of a measure. It guarantees the existence of a \emph{signed density} $f:\cX \to \RR$ such that $\mu(E)=\int_{E} f(x) dx$ for any $E \in \set{M}$.}
Specifically, when $\mu$ is a probability measure, then $f$ must satisfy $f(x)\geq 0$ for any $x\in \cX$ and $\int_{\cX} f(x) dx = 1$.
See \cite{durrett2019probability} and \cite{stein} for details on measure theory.

\subsection{ML Models as Selectivity Predictors} \label{subsec: learning}
We formulate selectivity estimation as an ML problem. A learning algorithm  $\mathcal{A}$ learns a model $M$ to predict query selectivity from a training set $\mathcal{W} = \{z_i = (q_i, l_i)\}_{i=1}^n$, comprising observed queries/ranges and their selectivities. $M$ minimizes the mean of loss $\ell$ over the dataset, where $\ell$ can be defined as the squared error $(q_i, l_i)$: $\ell = (M(q_i) - l_i)^2$, or the absolute error: $\ell = |M(q_i) - l_i|$. Additionally, Qerror~\cite{moerkotte2009preventing} (max$(\frac{M(q_i)}{l_i}, \frac{l_i}{M(q_i)}$) and Squared Logarithmic Error (\textit{e.g.,} SLE = $(\log{M(q_i)} - \log{l_i})^2$, which is equivalent to optimizing Qerror), are prevalent in the literature as it better captures errors on selective queries.

\subsection{PAC Learning Framework} \label{subsec: PAC}
Probably Approximately Correct (PAC) learning~\cite{kearns1994introduction} is a framework for mathematical and rigorous analysis of \emph{in-distribution} (In-Dist) generalization in machine learning. We first intuitively explain the high-level idea of PAC learnability. Readers who prefer a simpler explanation may directly refer to Table~\ref{tab:term.summary} for an intuitive summary of key concepts used in the paper.

\smallskip
\noindent \textcolor{black}{\textbf{PAC Learnability.} Consider a learner $\mathcal{A}$ that receives training samples $\{z_i\}$ from an unknown distribution $Q(z)$ and picks a \textbf{hypothesis} (or a function) $h$ from a \textbf{hypothesis space} or \textbf{function family} $\mathcal{H}$ (\textit{i.e.,} a family/class of selectivity functions in our scenario). In the classical PAC framework, $\mathcal{A}$ is assumed to efficiently find the best $h$.
 We say a function family $\mathcal{H}$ is \textbf{learnable} if, given \textbf{enough training data}, then \textbf{with high probability} $(1-\delta)$, the chosen function $h \in \mathcal{H}$ will have \textbf{low error} (no more than $\epsilon$) on unseen data from $Q(z)$. Importantly, $\mathcal{A}$ must succeed for any distribution $Q(z)$ and any choice of $\delta$, $\epsilon$. We then introduce two key theoretical results that can \textit{determine} whether a function family $\mathcal{H}$ (including $\{0,1\}$-valued and real-valued functions) is PAC learnable.
\begin{enumerate}[leftmargin=*]
    \item A class of $\{0,1\}$-valued functions is PAC learnable \textbf{if and only if} its VC dimension (will be introduced later) is finite.
    \item A real-valued function class is $\gamma$-learnable \textbf{if and only if} its $\gamma$-fat dimension (will be introduced later)  is finite.
\end{enumerate}
In practice, finding the absolute best hypothesis in $\mathcal{H}$ can be challenging (\textit{i.e.,} such efficient learner $\mathcal{A}$ may not exist); sometimes we only need a hypothesis whose true error is small, even if it is not the best. Fortunately, \textbf{uniform convergence} theory~\cite{alon1997scale} (or specifically, Chernoff bound ~\cite{hellman1970probability}) says that if $\mathcal{H}$ has a finite VC dimension (or $\gamma$-fat dimension), then with a \textbf{sufficiently large} training set, \textbf{every} hypothesis $h \in \mathcal{H}$ will have its \textbf{empirical error} close to its \textbf{true error} with high probability. Thus, we do not need the absolute best function; “what you see is what you get.” As a result, if 1) $\mathcal{H}$ is learnable and 2) our optimization finds a hypothesis $h \in \mathcal{H}$ with small training error, it follows that $h$ will have small true error. This paper adopts this definition of learnability.
}

\textcolor{black}{Below is a short introduction to the \emph{VC dimension} and the \emph{fat-shattering dimension}, which are measures of complexity for classification and real-valued function classes, respectively. We note that these are abstract mathematical concepts, and giving a fully rigorous treatment would exceed the scope of this paper. Since they appear only in our in-distribution generalization error theorem (Thm~\ref{theorem.remove.validity}) and our paper primarily focuses on OOD generalization, first-time readers may skip the two definitions if desired.}

\smallskip
\noindent\textcolor{black}{\textbf{Vapnik–Chervonenkis (VC) Dimension.}
 A function family $\mathcal{H}$ \textbf{shatters} a set of points if, for \textbf{every possible} way to assign $0/1$ labels to those points, $\mathcal{H}$ contains at least one function that matches those labels exactly. The \textbf{VC dimension} is the size of the largest set of points that can be shattered by $\mathcal{H}$.  
We also define the VC dimension \vcd{$\Sigma$} of a range space $\Sigma$ to be the size of the largest subset of $\mathcal{X}$ that can be shattered by $\Sigma$. The VC dimension of a range space of $d$-dimensional hyper-rectangles is 2$d$~\cite{kearns1994introduction}. }

\noindent\textcolor{black}{\textbf{Fat-Shattering Dimension.}
To handle \textbf{real-valued functions} (\textit{e.g.,} our selectivity functions), we use the \textbf{fat-shattering dimension}~\cite{kearns1994efficient}, which extends the VC dimension idea. Informally, a set of points is ``$\gamma$-shattered'' if the function class $\mathcal{H}$ can position those points \textbf{above or below some target values} by at least $\gamma$, matching any desired ``above/below'' pattern. The \textbf{$\gamma$-fat dimension} is how many points can be arranged this way.
We define the $\gamma$-fat shattering dimension \mbox{$\textrm{fat}_{{\mathcal{H}}}(\gamma)$} to be the size of the largest subset of $\mathcal{X}$ that can be $\gamma$-shattered by $\mathcal{H}$.}

\smallskip


\begin{table}
 \caption{Key concepts and their intuitive explanations.}
    \centering
    \scalebox{0.93}{
        \begin{tabularx}{0.5\textwidth}{c|X}
        \hline
        \textbf{Concept} &  \textbf{Intuitive Explanation}  \\
        \hline
      {\small Generalization}  & \footnotesize The model's capability to perform well on unseen queries that are \emph{not} in the training workload. \textcolor{black}{We can predict the outcomes on unseen data based only on training samples.}\\
       \hline
      \begin{minipage}[t]{0.27\linewidth} \small{Learnability/ In-Dist} \\ {\small Generalization} \end{minipage}  & \footnotesize Given sufficient training queries,  the model's true error on unseen queries drawn from the \emph{same distribution} with training queries is close to the training error.\\
        \hline
 \small{OOD Generalization} & \footnotesize  Given sufficient training queries, the model's true error on unseen queries drawn from a \emph{different distribution} from the training set is close to the training error. \\
        \hline
    \end{tabularx}}
    \label{tab:term.summary}
\end{table}

\noindent \textbf{Limitation of the PAC Learning Framework.} While PAC learnability  can be used to quantify the generalization error for hypothesis spaces with finite VC (or fat-shattering) dimension, they are applicable \emph{solely} to \emph{in-distribution generalization} where both training and test queries are drawn from the same distribution $Q(z)$. 


\vspace{-0.5em}
\subsection{Existing Theoretical Results ~\label{section.review}}
For self-containment, in this section, we briefly review the main learnability results of selectivity functions from the literature~\cite{hu2022selectivity}, and point out the important assumption made by the paper.

\noindent \textbf{Overview of~\cite{hu2022selectivity}.} Since selectivity functions are real-valued, to prove their learnability it suffices to show that their fat-shattering dimension is bounded. Using the same terminology in ~\cite{hu2022selectivity}, we cite the main Lemma ~\cite{hu2022selectivity}.

\begin{lemma} \label{lemma.k}
Consider a range space $\Sigma = (\mathcal{X}, \bm{\mathcal{R}})$ and the hypothesis class $\bm{\mathcal{S}}$ of range functions over input query ranges $R \in \bm{\mathcal{R}}$. For any $\gamma \in (0,1/2)$, the $\gamma$-fat shattering dimension of \set{S} is $\tilde{O}(\frac{1}{\gamma^{\lambda+1}})$\footnote{$\tilde{O}(\cdot)$ hides {polylogarihm dependencies on $1/\gamma$ for constant $\lambda$} }, where $\lambda$ is the \vcd{$\Sigma$} of the range space.
\end{lemma}

 \vspace{-0.5em} 
\noindent \textbf {Assumption.}
Note that the proof in~\cite{hu2022selectivity} relies on an important condition on the hypothesis class: every range function $S \in \set{S}$ is induced by a \emph{probability measure} via ~(\ref{eq.selfunc2}). 


\subsection{\mbox{The Gap Between Theory and Practice}~\label{section.gap}}
Recall from Section~\ref{section.prior} that there are four categories of query-driven approaches for learning selectivity functions: (1) linear and polynomial parametric functions (PFs), (2) histograms built from queries, (3) LEO, which can be seen as a combination of parametric functions and histograms, (4) deep learning (DL) models such as Multi-Set Convolutional Network (MSCN)~\cite{kipf2019learned}. Among them, deep learning models achieve the best practical performance. 
\textcolor{black}{{In this section, we theoretically analyze whether the probability measure assumption holds for these methods. We also conducted empirical experiments to verify our results; but we omit them here due to space constraints.}}

\noindent \textbf{Two necessary conditions.} One can show that if a learned selectivity function  $\hat{S}(R)$ (by a selectivity estimation model) is induced by a probability measure, it must satisfy \textbf{finite additivity} as well as \textbf{monotonicity} defined as follows.

\begin{itemize} [leftmargin=*]
  \item \textbf{Finite Additivity.} Implication of \ref{cond.add}, defined in \S~\ref{subsec: sgn_m}.
  
  \item \textbf{Monotonicity.} Let $R_1$ and $R_2$ be two union-compatible ranges over schema $\Sigma$, such that $R_1 \subseteq R_2$ for any instance of $\Sigma$.  We refer to this as a case of \emph{query containment}~\cite{abiteboul1995foundations}. Then finite additivity and positivity (\ref{cond.pos}) imply that
  $\hat{S}(R_1) \leq \hat{S}(R_2)$.

\end{itemize}

{\color{black}
\noindent \textbf{Theoretical Characteristics.} First, histograms satisfy both monotonicity and additivity \emph{by construction}. In contrast, PFs can violate both monotonicity and additivity due to \emph{negative parameters} and \emph{non-linear mappings} in the input and output of polynomial functions, respectively.  For LEO, while it maintains the additivity of histograms through the \emph{piecewise linear} form of its adjustment ratio $r$, it fails to ensure monotonicity because $r$ is not strictly increasing. This non-monotonicity means LEO is not derived from a probability measure. Like PFs, DL models also break both monotonicity and additivity due to \emph{negative weights} and \emph{non-linear mappings}. We summarize the analysis results in Table~\ref{table.check}.

Hence, unlike histograms (\textit{i.e.,} data models built from queries), other three regression-like approaches learn a direct mapping from query ranges to selectivities and are not guaranteed to be induced by probability measures. Therefore, they (including the best-performing deep learning models) do \emph{not} enjoy the theoretical results in \S~\ref{section.review}.
}

\begin{table}
\centering
 \caption{Theoretical Characteristics of query-driven methods.\label{table.check}}
 \vspace{-1em} 
\scalebox{1}{
\begin{tabular}{|c|c|c|c|c|}
  \hline
    & \makecell{PFs} & Histograms & LEO & \makecell{DL Models}  \\
  \hline
   Monotonicity & \ding{55} & \checkmark & \ding{55} & \ding{55} \\
  \hline
  Additivity & \ding{55} & \checkmark &  \textcolor{black}{\checkmark}  & \ding{55} \\
  \hline
\end{tabular}
}
\end{table}
 
\noindent \textbf{Problem Definition}.
We have shown that the selectivity functions learned by most query-driven models are not induced by probability measures, rendering the learnability results from previous work~\cite{hu2022selectivity} inapplicable. Despite this, these models, such as MSCN, exhibit impressive practical performance, outperforming histograms on several benchmarks~\cite{kipf2019learned}. Additionally, the PAC learning framework fails to characterize generalization error for OOD test workloads, which are prevalent in real-world scenarios.
Therefore, this paper aims to bridge the gap between theory and practice by 
 \circled{1}  relaxing the restrictions on the hypothesis class and deriving the corresponding PAC learnability results (\textbf{Goal 1}); 
 \circled{2}  exploring OOD generalization error beyond the PAC learning framework (\textbf{Goal 2});
 \circled{3} leveraging the theoretical results to design new strategies for improving existing selectivity learning models (\textbf{Goal 3}).

\vspace{-0.5em}
\section{A New Generalization Theory~\label{section.theory}}
In this section, we propose a new generalization theory that addresses the first two goals of the Problem Definition. Note that the proofs in \S~\ref{section.validity} and \S~\ref{section.ood} require advanced knowledge of measure theory and probability theory introduced in \S~\ref{section.preliminaries}. \textbf{Readers who prefer a simpler explanation may refer to \S~\ref{section.theory.summary}.}

\subsection{Learnability Under Signed Measures~\label{section.validity}}
We first demonstrate the learnability of the class of selectivity predictors induced by signed measures (\textit{i.e.,} removing restrictions \ref{cond.pos} and \ref{cond.one}). 
The results will be applied to \neucdf and LEO in Section ~\ref{section.neucdf} after showing that their hypothesis classes are indeed induced by signed measures.
 \subsubsection{\textbf{Learnability}}
 Given a range space $\Sigma = (\mathcal{X}, \set{R})$, let $\set{S}_{\sign}$ denote the hypothesis class that consists of all functions $\hat{S}:\set{R} \to \RR$ that are induced by signed measures absolutely continuous with respect to the Lebesgue measure. Recall the definition of $\mu_{\hS}$ in Section ~\ref{subsec: sgn_m}, and define the hypothesis class $\set{S}_{\sign}(C)$ for any $C \geq 0$ as follows. 
 \begin{equation*}
     \set{S}_{\sign}(C):= \cbr{\hat{S} \in \set{S}_{\sign}: \abs{\mu_{\hat{S}}} \leq C  }
 \end{equation*}
\noindent\fbox{\begin{minipage}{0.46\textwidth}
\begin{theorem} [\textbf{In-Distribution Generalization Error Bound}]  \label{theorem.remove.validity}
If \vcd{$\Sigma$} $= \lambda$ where $\lambda$ is some constant, then the fat-shattering dimension of $\set{S}_{\sign}(C)$ is finite and satisfies:
\begin{equation}
    \fatt{\set{S}_{\sign}(C)} = \tilde{O} \rbr{C\cdot(1/\gamma)^{\lambda+1}}
\end{equation}

\end{theorem}
\end{minipage}}

\smallskip

Theorem \ref{theorem.remove.validity} implies that $\set{S}_{\sign}(C)$ is $\epsilon$-learnable with a polynomial number of training queries in $\frac{1}{\epsilon}$ and $\log \frac{1}{\delta}$, following the arguments in~\cite{hu2022selectivity, bartlett1995more}. By \emph{uniform converge} theory, it means that given a sufficient number of training queries,  the training and true errors remain \emph{close}, \textit{i.e.,} $\rvert \er(\hat{S})-\er^{\text{train}}(\hat{S}) \rvert$ is \emph{bounded}.


\subsubsection{\textbf{Proof of Theorem~\ref{theorem.remove.validity}}}
Without loss of generality, assume $C=1$ since the general case follows from scaling. Set $\set{S}:=\set{S}_{\sign}(1)$ to be the hypothesis class. Following \cite{hu2022selectivity}, let $\cT \subset \set{R}$ be a subset $\gamma$-shattered by $\set{S}$ and partition $\cT$ based on the values of witnesses $\sigma(R)$:
\begin{equation*}
    \mathcal{T}_j= \cbr{R \in \mathcal{T}: \sigma(R) \in[(j-1) \cdot \gamma, j \cdot \gamma]}
\end{equation*}
for $j= -\ceil{1/\gamma}, -\ceil{1/\gamma}+1,...,0,...,\ceil{1/\gamma}-1, \ceil{1/\gamma}$. Let $k_j:= \abs{\cT_j}$.

First, Lemma 2.4 in \cite{hu2022selectivity} implies that there is an ordering of ranges in $\cT_j$, denoted by $\pi_j = \la R_1,...,R_{k_j} \ra$, such that for any probability distribution $D$ on $\cX$, we have
\begin{equation} \label{eqn: anyD}
    \EE_{x \sim D} I_x = O \rbr{k_j^{1-1/\lambda} \log k_j}
\end{equation}
where $I_x=\sum_{i=1}^{k_j-1}I_{i,x}$ and $I_{i,x}=\1(x \in R_{i}\oplus R_{i+1})$, $\oplus$ being the set symmetric difference.

Next, define the subset $E_j = \cbr{R_{2i} \mid 1\leq i \leq \flr{k_j/2}}$.
One can check that Lemma 2.2 in \cite{hu2022selectivity} still holds and ensures the existence of some $\hS_j \in \set{S}$ such that for any pair $R \in E_j$ and $R' \in \cT_j \setminus E_j$, we have
\begin{equation} \label{eq: gap_gamma}
    \hS_j(R) - \hS_j(R') > \gamma 
\end{equation}
With $\hS_j$ in hand, we define $\Delta_j$ according to whether $k_j$ is odd or even as follows.

\noindent \textit{$k_j$ is odd:}
\begin{align*}
    \Delta_j := & (\hS_j(R_2) - \hS_j(R_1)) + (\hS_j(R_2) - \hS_j(R_3)) + \cdots + \\
    & (\hS_j(R_{k_j-1}) - \hS_j(R_{k_j-2})) + (\hS_j(R_{k_j-1}) - \hS_j(R_{k_j}))
\end{align*}
\noindent \textit{$k_j$ is even:}
 \begin{align*}
    \Delta_j := &(\hS_j(R_2) - \hS_j(R_1)) + (\hS_j(R_2) - \hS_j(R_3)) + \cdots +\\
    &(\hS_j(R_{k_j-2}) - \hS_j(R_{k_j-3})) + (\hS_j(R_{k_j-2}) - \hS_j(R_{k_j-1})) +\\
    &(\hS_j(R_{k_j}) - \hS_j(R_{k_j-1})) 
\end{align*}
By definition of $\Delta_j$ above and (\ref{eq: gap_gamma}), one has
\begin{equation} \label{eqn: Del_low}
    \Delta_j \geq (k_j-1) \gamma 
\end{equation}

Since $\hS_j$ is induced by a signed measure $\mu_{\hS_j}$, denote by $\hf_j$ the signed density of $\mu_{\hS_j}$. Then one can show that $|\hf_j|$ is the density of $|\mu_{\hS_j}|$ and that $|\hf_j| / (|\mu_{\hS_j}|(\cX))) $ is the density of the probability measure $|\mu_{\hS_j}| / (|\mu_{\hS_j}|(\cX)))$ on $\cX$. Therefore, one can obtain
\begin{align}
    \Delta_j &\leq \int_{\cX} \abs{\hat{f}_j(x)} I_x dx \nonumber\\
    &= \abs{\mu_{\hS_j}}(\cX) \cdot \int_{\cX} \frac{\abs{\hat{f}_j(x)}}{\abs{\mu_{\hS_j}} (\cX)} \cdot I_x dx \nonumber\\
    &\overset{(i)}{=}  O \rbr{k_j^{1-1/\lambda} \log k_j} \label{eqn: del_upp}
\end{align}
Here (i) follows from (\ref{eqn: anyD}) and the assumption that $C=1$.

Finally, similar to \cite{hu2022selectivity}, one can combine (\ref{eqn: Del_low}) and (\ref{eqn: del_upp}) to show that $k_j = \tilde{O} \rbr{ (1/\gamma)^{\lambda}}$ and $|\cT| = \tilde{O} \rbr{ (1/\gamma)^{\lambda+1}}$. The proof is then complete.

\subsubsection{\textbf{Remark}} Although Thm~\ref{theorem.remove.validity} is a natural extension of prior work~\cite{hu2022selectivity}, it is crucial for developing a \emph{practical} theory. It applies to a broader array of selectivity predictors beyond the probability measures used previously~\cite{hu2022selectivity} (will be introduced in~\S \ref{section.neucdf}). Additionally, as will be presented in ~\S \ref{section.ood}, under mild assumptions, these predictors have bounded OOD generalization errors. This means we can predict their performance even when the test workload comes from a different distribution than the training workload, a common scenario in practice. \textbf{More importantly, proving OOD generalization is challenging, as it falls \emph{outside} the scope of the PAC learning framework. Therefore, existing results (\textit{e.g.,} fat-shattering dimension and results in ~\cite{hu2022selectivity}) within the PAC learning framework \emph{cannot} be reused.}

\subsection{OOD Generalization Error~\label{section.ood}}

In this section, we target the second goal in Problem Definition --- OOD generalization error beyond the PAC learning framework. \textbf{The main results appear in the callout for Theorem \ref{thm: ood}}.

The theorem shows that under the realizable assumption, 
 a predictor $\hat{S}$ trained with $n$ i.i.d. samples from a training distribution $Q$ to $(\epsilon, \delta)$-learn will have its generalization error on a different testing distribution $P$ bounded above by $C \sqrt{\epsilon}$ with probability at least $1 - \delta$, provided that Assumptions \ref{assump: bdd} through \ref{assump: samp_R} (introduced later in \S ~\ref{subsec: ood_theory}) hold.
As will be introduced in \S~\ref{section.neucdf}, this result will theoretically demonstrate the potential advantage of modeling CDFs over selectivities in terms of out-of-distribution generalization error.

The assumptions are relatively mild. Intuitively, Assumption \ref{assump: bdd} requires only that $\hat{S}$ is bounded and that the \emph{densities exist}; Assumption \ref{assump: samp_x} stipulates that the region covered by the testing distribution $P$ must be \emph{contained within} the region covered by the training distribution $Q$; and Assumption \ref{assump: samp_R} essentially requires sufficient \emph{diversity} in the training ranges.

\subsubsection{\textbf{Main Theoretical Results}}
\label{subsec: ood_theory}
It is important to note that one cannot expect an algorithm trained on a distribution $Q$ to generalize well to an arbitrary testing distribution ~$P$. To ensure provable and robust generalization, we impose the following assumptions.

\begin{assumption} \label{assump: bdd}
    The learned selectivity $\hat{S}$ is bounded such that there exists a constant $C_1$ for which $\abs{\hat{S}(R) } \leq C_1 $ for any $R \in \set{R}$. 
\end{assumption}
 
Before proceeding, we introduce some additional notations. Let $\set{Z} = \set{R} \times \RR$. We use $Q$ and $P$ to denote the training and testing distribution of $Z=(R,W) \in \set{Z}$, respectively. Given the training distribution $Q$, let $Q_R$ be the \emph{marginal distribution} of $R$ and define $\mathcal{X}_Q:= \bigcup_{R \in \supp Q_R} R$, which is a subset of $\cX$. The marginal distribution $P_R$ and the set $\mathcal{X}_{P}$ are defined similarly for the testing distribution.
We now introduce Assumption \ref{assump: samp_x} and \ref{assump: samp_R}.

\begin{assumption} \label{assump: samp_x}
    There exists a constant $ C_2$ such that the marginal training and testing distributions $Q_R$ and $P_R$  satisfy
    \begin{align*}
        P_{R \sim P_R} \sbr{x \in R} \leq C_2 \cdot  P_{R \sim Q_R} \sbr{x \in R}, \quad \forall x \in \mathcal{X}.
    \end{align*}
\end{assumption}

\noindent\textit{Remark.} Assumption \ref{assump: samp_x} requires that the probability $P_{R \sim P_R} \sbr{x \in R}$ (the likelihood of $x$ being sampled during testing) is upper-bounded by the probability $P_{R \sim Q_R} \sbr{x \in R}$ (the likelihood of $x$ being sampled during training) multiplied by a constant $C_2$. This implies that $\cX_P \subset \cX_Q$. The rationale is that if $\cX_P$ includes some $x$ that is not covered by any range during training, then one cannot expect to learn the selectivity around $x$ accurately.

\begin{assumption} \label{assump: samp_R}
    The true $S_D$ and the learned selectivity $\hS$ are induced by signed measures that are absolutely continuous, with corresponding signed densities $f_D,\hf$. Additionally,
    there exists a constant $ c_3>0 $ such that $Q_R$ and the signed density $\hat{f}$ satisfy
    \begin{align}
        &\EE_{R \sim Q_R} \abs{\int_{\cX} \rbr{\hat{f}(x)-f_D(x) } \1(x \in R)dx } \nonumber\\
        \geq c_3 \cdot &\EE_{R \sim Q_R} \int_{\cX} \abs{ \hat{f}(x)-f_D(x) } \1(x \in R)dx  \nonumber
    \end{align}
\end{assumption}

\noindent\textit{Remark.} Assumption \ref{assump: samp_R} presupposes the validity of interchanging the order of integration and the absolute value. Intuitively, it ensures that $Q_R$ covers a diverse set of ranges rather than focusing on ranges where the error $\hat{S}(R) - S_D(R)$ happens to be relatively small. A simple example illustrating a situation where Assumption \ref{assump: samp_R} holds is provided below in Example \ref{ex}.

\begin{example} \label{ex}
    For $\mathcal{X} = [-1/2, 1/2]$, suppose the densities $f_D$ and $\hat{f}$ are defined as $f_D(x) = 1$ and $\hat{f}(x) = 1 + 2\delta_n x$, where $\delta_n$ is a parameter that quantifies how well $\hat{f}$ approximates $f_D$. If $Q_R$ is uniformly distributed over intervals of length $1/4$ with centers located within the range $[-3/8, 3/8]$, then it can be verified by direct computation that Assumption \ref{assump: samp_R} is satisfied with $c_3 = 1/2$ for any value of $\delta_n$.
\end{example}

\noindent Now, we are ready to present our OOD generalization error bound:

\medskip

\noindent\fbox{\begin{minipage}{0.46\textwidth}
\begin{theorem} [\textbf{OOD Generalization Error Bound}] \label{thm: ood} 
 Suppose Assumption \ref{assump: bdd}-\ref{assump: samp_R} hold. In addition, if the in-distribution generalization error of $\hat{S}$ can be bounded by
    \begin{align}
        P_{Z_1^n \sim Q^{\otimes n}} \sbr{\er_Q(\hat{S}) < \epsilon } \geq 1-\delta
        \label{eqn: ood_in_samp_err}
    \end{align}
    then the out-of-distribution generalization error $\er_P(\hat{S})$ satisfies
    \begin{align}
        P_{Z_1^n \sim Q^{\otimes n}} \sbr{\er_P(\hat{S}) < \frac{(C_1 +1)C_2}{c_3} \sqrt{\epsilon} } \geq 1-\delta
        \label{eqn: ood_out_samp_err}
    \end{align}
\end{theorem}
\end{minipage}}

\medskip
Informally, the Theorem states that if Assumptions \ref{assump: bdd}-\ref{assump: samp_R} hold and $\hat{S}$ achieves bounded in-distribution generalization error, then $\hat{S}$ will also have bounded out-of-distribution generalization error.

We then present the full proof. To better understand the proof, consider the following sequence of inequalities:
\begin{align*}
    \er_P(\hat{S}) &\overset{(a)}{\lesssim} \EE_{R \sim P_R} \abs{\hat{S}(R) - S_D(R)} \\
    &\overset{(b)}{\lesssim} \EE_{R \sim Q_R} \abs{\hat{S}(R) - S_D(R)}  \overset{(c)}{\leq} \sbr{\er_Q(\hat{S})}^{1/2}
\end{align*}
Here $a_n \lesssim b_n$ means $a_n = O (b_n)$. Step (a) follows from the upper-bound Assumption \ref{assump: bdd}; (b) involves a change of measure from $P_R$ to $Q_R$ and  connects through Assumptions \ref{assump: samp_x} and \ref{assump: samp_R}; and (c) is based on the  Cauchy–Schwarz inequality. The complete proof provides a detailed justification for each of these inequalities.
\begin{proof}
    To bound $\er_P(\hat{S})$ in (\ref{eqn: ood_out_samp_err}), note that one has
    \begin{align}
        \er_P(\hat{S}) &= \EE_{R \sim P_R} \rbr{\hat{S}(R) - S_D(R)}^2 \nonumber\\
        &\overset{(i)}{\leq} (C_1+1) \cdot \EE_{R \sim P_R} \abs{\hat{S}(R) - S_D(R)} \nonumber\\
        &\leq (C_1+1) \cdot \int_{\mathcal{X}} \abs{\hat{f}(x) - f_D(x) } \cdot \sbr{\EE_{R \sim P_R} \1 \rbr{x \in R}} \cdot dx \nonumber\\
        &\overset{(ii)}{\leq} (C_1+1) \cdot C_2 \underbrace{\int_{\mathcal{X}} \abs{\hat{f}(x) - f_D(x) } \cdot \sbr{\EE_{R \sim Q_R} \1 \rbr{x \in R}} \cdot dx }_{\mathbf{(I)}}
        \label{eqn: prf_ood_test_up}
    \end{align}
    Here (i) follows from Assumption \ref{assump: bdd}; (ii) is due to Assumption \ref{assump: samp_x}.

    Meanwhile, one can obtain that
    \begin{align}
        \mathbf{(I)} &\overset{(i)}{=} \EE_{R \sim Q_R} \int_{\cX} \abs{\hat{f}(x) - f_D(x) }\1 \rbr{x \in R} dx \nonumber \\
        &\overset{(ii)}{\leq} c_3^{-1} \EE_{R \sim Q_R} \abs{\int_{\cX} \rbr{\hat{f}(x) - f_D(x) }\1 \rbr{x \in R} dx} \nonumber \\
        &= c_3^{-1} \underbrace{\EE_{R \sim Q_R} \abs{\hat{S}(R) - S_D(R)}}_{\mathbf{(II)}}
        \label{eqn: prf_ood_abs_low}
    \end{align}
    
        
    Here (i) follows from the Fubini's theorem, and (ii) is a result of Assumption \ref{assump: samp_R}.

    Note that the Cauchy–Schwarz inequality implies that
    \begin{align*}
        \mathbf{(II)}\leq \sbr{\EE_{R \sim Q_R} \rbr{\hat{S}(R) - S_D(R)}^2}^{1/2} = \sbr{\er_Q(\hat{S})}^{1/2}
    \end{align*}
    Combine the above inequality with (\ref{eqn: ood_in_samp_err}) implies
    \begin{align}
        P \cbr{\mathbf{(II)} < \sqrt{\epsilon}} \geq 1-\delta
        \label{eqn: prf_ood_abs_up}
    \end{align}
    
        


    Finally, combining (\ref{eqn: prf_ood_test_up}), (\ref{eqn: prf_ood_abs_low}) and (\ref{eqn: prf_ood_abs_up}) gives (\ref{eqn: ood_out_samp_err}).
\end{proof}

\subsubsection{\textbf{OOD Scenarios}\label{section.ood.cases}}
We define three specific OOD scenarios
which naturally arise in real-world applications.


\noindent \textbf{Scenario 1: Query Center Move} refers to a shift in the predominant focus of queries, characterized by a change in the attribute values around which the queries are \emph{concentrated}.

\begin{example}[Center Move]
    $\cX=\RR$. Both training and test distribution $Q_R,P_R$ are supported on intervals of length $2$. For training distribution $Q_R$, the center of the interval is uniform on $[0,1]\cup[1,2]$ while for test distribution $P_R$, the center of the interval is uniform on $[1,2]$. One can check that Assumption \ref{assump: samp_x} holds with $C_2=2$.
\end{example}

\noindent \textbf{Scenario 2: Query Granularity Shift} refers to a change in the granularity of query selection predicates. Granularity pertains to the \emph{specificity} or \emph{broadness} of the data subsets accessed by queries.

\begin{example}[Granularity Shift]
    $\cX=\RR$. The training distribution $Q_R$ is supported on intervals of fixed length $1$ with center uniformly distributed on $[-2,3]$, while the test distribution $P_R$ is supported on intervals of fixed length $2$ with center uniformly distributed on $[0,1]$. One can check that Assumption \ref{assump: samp_x} holds with $C_2 = 5$.
\end{example}

{\color{black}
\noindent \textbf{Scenario 3: Query Structure Change.} When the join graph remains unchanged, adding or dropping predicates essentially changes the query granularity, therefore reducing to Scenario 2. Our theory does not support changes in the join graph as both Thm~\ref{theorem.remove.validity} and Thm~\ref{thm: ood} hold for each join graph \emph{independently}. Developing unified error bounds for all join graphs is promising future work.

{\color{black}
\subsubsection{\textbf{Point Queries}\label{section:point_queries}}
Following~\cite{dutt2019selectivity, dutt2020efficiently}, we can treat point queries as range queries that cover only the corresponding data point. For instance, point query ($\texttt{t.production\_year=1980}$) can be rewritten as $(\texttt{t.production\_year>1979} \wedge \texttt{t.production\_year} \leq 1980)$. This approach makes our results applicable to point queries.

\subsubsection{\textbf{Data Distribution Shifts}\label{section:data_shift}} 
Our theory assumes a static data distribution, and shifts in data distribution can introduce additional errors. Although handling data distribution shifts is not the primary focus of this paper, we provide a preliminary theoretical result that extends Theorem \ref{thm: ood} to account for such shifts:
Let \(\tv{Q_X}{P_X} \coloneqq \sup_{x\in\set{X}}\bigl|Q_X(x) - P_X(x)\bigr|\) be the \emph{total variation distance} between the old data distribution \(Q_X\) and the new data distribution \(P_X\). 


\begin{proposition} [\textbf{OOD Generalization Error Bound with Data Shifts}] \label{proposition:data_shift}
 Suppose assumptions in Theorem \ref{thm: ood} hold. In addition, if $\tv{Q_X}{P_X}\leq \epsilon_{TV}$, and the in-distribution generalization error of $\hat{S}$ can be bounded by $P_{Z_1^n \sim Q^{\otimes n}} \sbr{\er_Q(\hat{S}) < \epsilon } \geq 1-\delta$, then the out-of-distribution generalization error $\er_P(\hat{S})$ satisfies
    \begin{align}
        P_{Z_1^n \sim Q^{\otimes n}} \sbr{\er_P(\hat{S}) < (C_1 +1)C_2 \rbr{c_3^{-1}\sqrt{\epsilon} + 2 \epsilon_{TV}} } \geq 1-\delta
    \end{align}
\end{proposition}

\begin{proof}
    Following similar steps as in Theorem \ref{thm: ood}, we apply triangle inequality to (\ref{eqn: prf_ood_test_up}) to obtain
    \begin{align*}
        \abs{\hat{f}(x) - f_{P_X}(x) } \leq \abs{\hat{f}(x) - f_{Q_X}(x)} + \abs{f_{Q_X}(x) - f_{P_X}(x)}
    \end{align*}
    Finally, noticing that $\tv{Q_X}{P_X} = 2^{-1}\int \abs{f_{Q_X}(x) - f_{P_X}(x)} dx$ gives the desired result.
\end{proof}

Intuitively, this proposition says that if $P_X$ and $Q_X$ are close in terms of the total variation distance, then the OOD generalization error for $\hat{S}$ remains bounded. However, we note that relying on $TV$ as the measure of data distribution shifts is a strong assumption. Future research might explore more flexible metrics for studying how OOD generalization error behaves under diverse data shifts.


}
}

\subsection{Summary and Discussion~\label{section.theory.summary}}

Combining the results from \S~\ref{section.validity} and \S~\ref{section.ood}: 

\noindent\fbox{\begin{minipage}{0.46\textwidth}
 \textbf{Summary of our results}: If a selectivity learning model is induced by a \textit{signed measure} and trained on a sufficient number of queries, both its in-distribution generalization error (when training and test queries are from the same distribution) and out-of-distribution generalization error (when training and test queries come from different distributions) are \textit{bounded} (\textcolor{black}{\emph{i.e.,} close to training error}), under \emph{}{mild} assumptions (Assumption \ref{assump: bdd}-\ref{assump: samp_R}).
\end{minipage}}

\smallskip
The summary provides insights for designing improvement strategies for query-driven models: if we can show a class of selectivity learning models that are \emph{provably} induced by signed measures, then the favorable in-distribution (Theorem~\ref{theorem.remove.validity}) and OOD (Theorem~\ref{thm: ood}) generalization results are immediately applicable.

\vspace{-0.5em}
\section{Modeling \emph{CDF}s with \emph{Neural Net}s~\label{section.neucdf}}
Building on the insights in \S~\ref{section.theory.summary}, this section introduces the first strategy for enhancing query-driven selectivity learning models.
We first consider this question: is it feasible to design a selectivity estimation paradigm that works well in practice \emph{and} inherits our theoretical guarantees?
Unfortunately, this is not easy. Although existing approaches developed from SOTA theory~\cite{hu2022selectivity} achieve SOTA results \emph{among selectivity predictors that are induced by probability measures}, they are
not as effective (experimentally shown in \S~\ref{section.eval}) as recent deep learning-based models in practice due to the limited model capacities of models induced by probability measures.

On the other hand, deep learning models, while lacking comprehensive theoretical backing, demonstrate remarkable efficacy in practice. Often, deep models can achieve both very small training and test errors when queries are drawn from the same distribution, which cannot be fully explained by existing theories like the PAC framework. This aligns with extensive well-known empirical evidence in the ML literature (see~\cite{zhang2021understanding} for an overview).

Therefore, an ambitious goal is to combine the theoretical results from previous sections and the practicality of neural nets, so that the new selectivity estimation paradigm enjoys both theoretical guarantees and practical utility. In pursuit of this, we propose a novel selectivity estimation paradigm/framework, \neucdf.

\subsection{Overview~\label{section.neurocdf.overview}}
\noindent \textbf{High-Level Idea.} 
\neucdf leverages the fact \emph{that the selectivity of a rectangular query can be computed as a linear combination of the CDFs evaluated at its vertices.} (will be discusses in ~\S~\ref{section.neucdf.convert}). 
CDFs, in statistical terms, measure the probability that a random variable takes a value less than or equal to a specific point.
Therefore, the key idea of \neucdf is that, instead of directly modeling the ultimate selectivities of input queries, \emph{we use a neural network as the model to parameterize the underlying CDFs.} The query selectivity can be estimated by multiple calls to the \emph{CDF prediction model} $\mathcal{M}$ and aggregating the results, as shown in Figure~\ref{fig.neucdf}. 

\begin{figure}
\centering
\begin{subfigure}{.18\textwidth}
\captionsetup{justification=centering}
\begin{tikzpicture}[scale=0.5, transform shape]
\tikzstyle{every node}=[font=\LARGE]
        \draw[thick,->] (0,0) -- (5,0) node[anchor=north west] (mark1) {$x_1$};
        \draw[thick,->] (0,0) -- (0,5) node[anchor=south east]  {$x_2$};
        
        \filldraw[fill=blue!20, draw=blue, thick] (2,2) rectangle (4,4);
        
        \node at (2,2) [below left] {$F(a_1, a_2)$};
        \node at (4,4) [above right] {$F(b_1, b_2)$};
        \node at (2,4) [above left] {$F(a_1, b_2)$};
        \node at (4,2) [below right] {$F(b_1, a_2)$};
        
        \draw[dotted] (2,0) node[below]{$a_1$} -- (2,2);
        \draw[dotted] (4,0) node[below] {$b_1$} -- (4,4);
        \draw[dotted] (0,2) node[left] {$a_2$} -- (2,2);
        \draw[dotted] (0,4) node[left] {$b_2$} -- (4,4);

        \node[below=0.2cm of mark1, align=center, xshift=-2.5cm] (desc1) {\textbf{Query $q$ to CDFs ($F(\cdot)$)}};
    \end{tikzpicture}
\end{subfigure}
\begin{subfigure}{.29\textwidth}
\captionsetup{justification=centering}
\begin{tikzpicture}[node distance=2cm, auto, >=Stealth, thick, scale=0.6, transform shape]

    \node[draw, rectangle, fill=blue!5] (query1) {query $q$};
    \node[draw, rectangle, trapezium left angle=75, trapezium  right angle=105, fill=blue!15, right=2cm of query1] (model1) {model $\mathcal{M}$};
    \node[draw, ellipse, right=2.2cm of model1, fill=pink!10] (output1) {$\hat{S}(q)$};
    
    \draw[->] (query1) -- (model1);
    \draw[->] (model1) -- (output1);

    \node[below=0.2cm of model1, align=center] (desc1) {\textbf{Direct Selectivity Modeling}};

    \node[draw, rectangle, fill=blue!5, below=2.7cm of query1] (query2) {query $q$};
    
    \node[draw, ellipse, fill=white, right=0.5cm of query2, yshift=1.5cm, minimum width=0.8cm, minimum height=0.5,font=\small] (cdf1) {$(b_1,b_2)$};
    \node[draw, ellipse, fill=white, right=0.5cm of query2, yshift=0.5cm, minimum width=0.8cm, minimum height=0.5,font=\small] (cdf2) {$(a_1,b_2)$};
    \node[draw, ellipse, fill=white, right=0.5cm of query2, yshift=-0.5cm, minimum width=0.8cm, minimum height=0.5,font=\small] (cdf3) {$(b_1,a_2)$};
    \node[draw, ellipse, fill=white, right=0.5cm of query2, yshift=-1.5cm, minimum width=0.8cm, minimum height=0.5,font=\small] (cdf4) {$(a_1,a_2)$};
    
    \node[draw, rectangle, trapezium left angle=75, trapezium stretches=true, minimum height=1.5cm, fill=blue!15, right=2.6cm of query2] (model2) {$\mathcal{M}$};
    
    \node[draw, ellipse, fill=white, right=0.1cm of model2, yshift=1.5cm, minimum width=0.5cm, minimum height=0.2,font=\small] (f1) {$\hat{F}(b_1,b_2)$};
    \node[draw, ellipse, fill=white, right=0.1cm of model2, yshift=0.5cm,minimum width=0.5cm, minimum height=0.2,font=\small] (f2) {$\hat{F}(a_1,b_2)$};
    \node[draw, ellipse, fill=white, right=0.1cm of model2, yshift=-0.5cm,minimum width=0.5cm, minimum height=0.2,font=\small] (f3) {$\hat{F}(b_1,a_2)$};
    \node[draw, ellipse, fill=white, right=0.1cm of model2, yshift=-1.5cm,minimum width=0.5cm, minimum height=0.2,font=\small] (f4) {$\hat{F}(a_1,a_2)$};
    
    \node[draw, ellipse, fill=white, right=2.5cm of model2, fill=pink!10] (output2) {$\hat{S}(q)$};

    \draw[->] (query2) -- (cdf1.west);
    \draw[->] (query2) -- (cdf2.west);
    \draw[->] (query2) -- (cdf3.west);
    \draw[->] (query2) -- (cdf4.west);

    \draw[->] (cdf1.east) -- (model2);
    \draw[->] (cdf2) -- (model2);
    \draw[->] (cdf3) -- (model2);
    \draw[->] (cdf4.east) -- (model2);

    \draw[->] (model2.north east) -- (f1);
    \draw[->] (model2) -- (f2);
    \draw[->] (model2) -- (f3);
    \draw[->] (model2.south east) -- (f4);

    \draw[->] (f1) -- (output2) node[midway, above ] {\textbf{+}};
    \draw[->] (f2) -- (output2) node[midway, above ] {\textbf{-}};
    \draw[->] (f3) -- (output2) node[midway, below] {\textbf{-}};
    \draw[->] (f4) -- (output2) node[midway, below ] {\textbf{+}};

    \node[below=1.3cm of model2, align=center] (desc2) {\textbf{NeuroCDF: Modeling CDFs for Selectivity Estimation}};
    
\end{tikzpicture}
\end{subfigure}
\caption{Left: relationship between a rectangle query and CDFs; Right: direct selectivity modeling \textit{v.s.} \neucdf.} \label{fig.neucdf}
\end{figure}

\noindent \textcolor{black}{\textbf{Theoretical Guarantees.} As will be discussed in \S~\ref{section.neurocdf.theory}, \neucdf, as a framework, can be proved to be induced by signed measures through its CDF modeling. Hence, both the in-distribution (Thm~\ref{theorem.remove.validity}) and OOD (Thm~\ref{thm: ood}) generalization error bounds directly apply to \neucdf. This means given sufficient training queries and under Assumption \ref{assump: bdd}-\ref{assump: samp_R}, the in-distribution and OOD generalization errors of \neucdf are both close to its training error --- an advantage not present in existing methods that directly model selectivities.
}

{\color{black}
Apart from the theoretical guarantees provided by \neucdf, \neucdf combines the empirical strengths of neural nets as NNs are known for achieving very low training error~\cite{he2016deep} due to their high model capacity.}
Note that \neucdf does not offer generalization theories for neural networks \emph{per se}, but it leverages their empirical success (\textit{e.g.,} low training error) alongside the formal guarantees of the CDF modeling paradigm permitted by our theory.

\smallskip
{\color{black}
\noindent \textbf{Workflow of \neucdf.} In a \(d\)-dimensional data space, the CDF prediction model $\mathcal{M}$ of \neucdf takes as input a vector \(\mathbf{x} = [x_1, x_2, \ldots, x_d]^\top\) of real-valued variables and 
outputs an estimated cumulative distribution function (CDF), 
\(\hat{F}(\mathbf{x}) = \hat{P}(X \le \mathbf{x})\). With a query workload $\mathcal{W} = \{(q, l)\}$, the \neucdf framework proceeds in four steps, beginning with two data preprocessing phases.

\noindent \circled{1} \textbf{Normalization.} Each range query \(q\) is a \(d\)-dimensional hyper-rectangle 
\(\bigl(\tilde a_1 < x_1 \le \tilde b_1\bigr) \wedge \cdots \wedge \bigl(\tilde a_d < x_d \le \tilde b_d\bigr)\).
We apply min-max normalization to scale all unnormalized $\tilde a_i, \tilde b_i$ into normalized values  $a_i, b_i \in [0,1]$.
Unqueried attributes/dimensions (\textit{i.e.,} attribute not involved in the query) are set to \([0,1]\).

\noindent \circled{2} \textbf{CDF Conversion (\S~\ref{section.neucdf.convert})}. 
Each range query $q$ is converted into a set of vectors $v$ (at each vertex of $q$), which serve as inputs for training the CDF prediction model $\mathcal{M}$.

\noindent \circled{3} \textbf{Model Training}. Training of \neucdf uses forward--backward propagation with mean squared error (MSE) loss. Unlike existing query-driven models which directly predict the query selectivity using only one forward pass, \neucdf computes a query's 
selectivity \(\hat{S}(q)\) by gathering multiple CDF values \(\hat{F}(v)\) via 
\(\mathcal{M}\). These values are combined using Equation~\eqref{eq.cdf}. 
Because the entire procedure is fully \emph{differentiable}, the loss 
\(\sum (\hat{S}(q) - S(q))^2\) can be optimized via stochastic gradient descent (SGD) and batched training.


\noindent \circled{4} \textbf{Prediction.}  Once trained, \neucdf uses the same multi-call forward process: for an incoming query \(q\), it computes each required \(\hat{F}(v)\), then aggregates them to derive \(\hat{S}(q)\).
}



\subsection{Converting Queries to CDFs~\label{section.neucdf.convert}}  
Consider the case of 2-dimensional data shown in Figure~\ref{fig.neucdf} (left), one can verify that the selectivity of a query $q:\{(a_1 < x_1 \leq b_1) \land (a_2 < x_2 \leq b_2)\}$ (represented by the rectangle in blue) can be computed by aggregating the CDF values at the four vertices \textcolor{black}{(\textit{i.e.,} $(b_1, b_2), (a_1, b_2), (b_1,a_2), (a_1, a_2)$)} of the query rectangle,
\begin{equation*}
    S(q) = F(b_1, b_2) - F(a_1, b_2) - F(b_1, a_2) + F(a_1, a_2).
\end{equation*}
We extend the formula to $d-$dimensional data (outlined on page 197 of the book~\cite{durrett2019probability}). Let $q$ be a range query (of hyper-rectangle) in the $d$ dimensional space, \textit{i.e.,} $q=(a_1, b_1] \times ... \times (a_d, b_d]$. The vertices $V$ of this hyper-rectangle are $V=\{a_1, b_1\} \times ... \times \{a_d, b_d\}$. For any vertex $v \in V$, define $\#a(v)$ as the number of $a$'s in $v$, indicating the count of left endpoints. For example, in Figure~\ref{fig.neucdf} (left), $\#a([a_1, b_2])=1$.
The general case formula is subsequently provided for completeness.

\begin{theorem}
Let $\text{sgn}(v) = (-1)^{\#a(v)}$. The selectivity of range query $q$ in $d-$dimensional space is computed by aggregating the CDF values at all vertices of the query hyper-rectangle using the below formula,
\begin{equation}\label{eq.cdf}
    S(q) = \sum_{v \in V} \text{sgn}(v)F(v)
\end{equation}
\end{theorem}
\begin{sproof}
This theorem is a direct application of the inclusion-exclusion principle~\cite{roberts2009applied}. See page 36 of the book~\cite{durrett2019probability} for details.
\end{sproof}

\noindent \textbf{Model Choice for CDF Prediction.}  \neucdf can incorporate any query-driven model architecture by viewing the input vector as a query (\textit{i.e.,} we can interpret every CDF as the selectivity estimate of a one-sided query). For example, $F(b_1, b_2)$ is equivalent to the selectivity estimation of a legitimate query $q:\{( x_1 \leq b_1) \land ( x_2 \leq b_2)\}$. Thus, possible model choices include Multi-Set Convolution Networks~\cite{kipf2019learned}, MLP with flattened query encoding~\cite{dutt2019selectivity}, or more recent NN models~\cite{reiner2023sample, li2023alece, sun2019end}.
{\color{black} 
Although non-NN regression methods such as XGBoost~\cite{chen2016xgboost} offer greater interpretability and could be used for CDF prediction in the \neucdf framework, we opt for NNs due to their \emph{ease of optimization} in our setting. XGBoost relies on direct mappings from data points to their CDF values (\(\mathbf{x} \mapsto F(\mathbf{x})\)) and optimizes based on gradients 
between predictions and actual values. However, in selectivity learning, we \emph{only} have mappings from queries to their selectivities (\(q \mapsto S(q)\)), lacking the direct data-to-CDF mappings (or actual CDF values) required by XGBoost. This makes optimizing XGBoost challenging. NNs, on the other hand, can be trained \emph{end-to-end} effectively using only \(q \mapsto S(q)\) mappings. Because the computation in Eq.~\ref{eq.cdf} is fully differentiable, we can employ backpropagation without needing direct \(\mathbf{x} \mapsto F(\mathbf{x})\) mappings (or the actual values of  $F(\mathbf{x})$).
}

\subsection{\textcolor{black}{Efficiency}\label{section.neucdf.efficiency}}

\begin{theorem}
The number of calls to the CDF prediction model $\mathcal{M}$ for estimating a query selectivity is $2^{n_c}$, where $n_c$ is the number of attributes/columns involved in the query.
\end{theorem}
\begin{proof}
\textcolor{black}{
First, every unqueried attribute has its \(a_i=0, b_i=1\). By definition, if any \(v_i=0\), it directly implies \(F(v)=0\), eliminating the need for CDF estimations. Consequently, those CDFs requiring estimates from \(\mathcal{M}\) will always have their unqueried attribute \(v_i=1\). Hence, the number of distinct \(v\)-vectors requiring \(\mathcal{M}\)'s estimates is \(2^{n_c}\), where \(n_c\) is the number of columns involved in the query.
}
\end{proof}
\textcolor{black}{
The result shows that in \neucdf, we do not have to estimate the CDF value for every possible vertex \(v\). Only those \textit{necessary} vertices require estimates from \(\mathcal{M}\).
}


\subsection{Theoretical Analysis~\label{section.neurocdf.theory}} {Next, we prove that \neucdf, as a framework, is induced by a signed measure due to its CDF modeling paradigm. \emph{This connects \neucdf to the two theoretical results in previous sections}. Surprisingly, this property applies to LEO as well. 

\begin{theorem} \label{thm: induce}
    Suppose $\set{R}$ consists of axis-aligned hyper-rectangles. Given a function $\hS:\set{R} \to \RR$, suppose there exists a function $F_{\hS}: \cX\to \RR$ such that for any $R \in \set{R}$, $\hS(R) = \sum_{v \in V_R} \text{sgn}(v)F_{\hS}(v)$ where $V_R$ is the vertex set of $R$. Then $\hS$ is induced by a signed measure.
\end{theorem}

\begin{sproof}
    This can be shown by a simple modification of the proof of Theorem 1.1.11 in ~\cite{durrett2019probability}.
\end{sproof}

\begin{corollary} \label{cor: induce}
    All predictions from \neucdf and LEO are induced by signed measures.
\end{corollary}

\begin{sproof}
    \neucdf and LEO satisfy Theorem \ref{thm: induce} by design, as they model CDFs (for LEO, $\hat{F}_{LEO}(x) = F_{\mathrm{hist}}(x) \cdot g_{\mathrm{adjust}}(x)$ where $F_{\mathrm{hist}}(x) $ is the CDF of the histogram and $g_{\mathrm{adjust}}(x)$ is the adjustment factor at $x$) and use Eq~\ref{eq.cdf} to estimate range queries.
\end{sproof}
\vspace{-0.6em}

With Corollary \ref{cor: induce} in place, let $\set{S}_{\neucdf}$ and $\set{S}_{LEO}$ denote the hypothesis class of \neucdf and LEO when the inducing signed measures are all absolute continuous. Then 
the learnability results for \neucdf and LEO are given as follows. 

\begin{theorem} \label{thm: neucdl_leo}
Let $\Sigma = (\mathcal{X}, \set{R})$ be a range space. If  \vcd{$\Sigma$} $= \lambda$ where $\lambda$ is some constant,
then the fat-shattering dimension of $\set{S}$ is finite and satisfies: $\fatt{\set{S}} = \tilde{O}((1/\gamma)^{\lambda+1})$ for any $\set{S} \in \cbr{\set{S}_{\neucdf} \ , \set{S}_{LEO}}$.
\end{theorem}

\vspace{-0.6em}
\begin{sproof}
    One can show that the predictions of \neucdf and LEO are bounded, and hence $\set{S} \subset \set{S}_{\sign}(C)$ for some constant $C$. Then the theorem follows by applying Theorem \ref{theorem.remove.validity}.
\end{sproof}
\vspace{-0.6em}

\noindent \textbf{Limitation of \neucdf.} Currently, \neucdf is not compatible with Qerror or MSLE because it can yield negative estimates where Qerror does not apply. This issue arises as the NN model might fail to produce a valid CDF, which can lead to negative values in estimates from ~(\ref{eq.cdf}).
We attempted to address this issue by clipping negative estimates to a small value (e.g., $1/|D|$) or enforcing monotonicity~\cite{liu2020certified}. Unfortunately, we observed significant performance degradation in practice since 1) the clipping is not differentiable preventing the model from learning from queries with clipped estimates; 2) the enforcement of monotonicity would reduce model capacity and introduce noises into training. We leave training \neucdf with Qerror as future work.

\subsection{\mbox{Preliminary Evaluation of \neucdf~\label{neucdf.evaluation}}}
We implement it with LW-NN~\cite{dutt2019selectivity} and MSCN~\cite{kipf2019learned} to validate our improvement strategy (\textit{i.e., CDF modeling}). Here we intentionally exclude data information to concentrate on the modeling paradigm itself. Despite that \neucdf is not compatible with Qerror, a major loss function used in recent query-driven models, we observe significant improvement in OOD generalization on both models, which further inspires us to design a more general improvement strategy in the next section.

We generate a collection of training queries on a synthetic dataset sampled from a 10-dimensional highly correlated Gaussian distribution. Moreover, we use two types (In-distribution and OOD) of test queries to assess the model generalization capabilities.

\begin{table}
\caption{\mbox{Generalization performances of different models}\label{table.qualitative.compare}}
\scalebox{0.9}{
\begin{tabular}{|c|c|c|c|c|}
  \hline
  \multirow{2}{*}{Model}  &\multicolumn{2}{c|}{In-Dis Generalization}  & \multicolumn{2}{c|}{OOD Generalization}  \\
\cline{2-5}
  & RMSE & Qerror &   RMSE  & Qerror   \\
    \hline
  {LW-NN}  & \ding{72}\ding{73} & \ding{72}\ding{72} & \ding{73}\ding{73}  & \ding{73}\ding{73}  \\
  \hline
{MSCN}  & \ding{72}\ding{73} & \ding{72}\ding{72} & \ding{73}\ding{73}  & \ding{73}\ding{73}  \\
   \hline
   \makecell{\neucdf (LW-NN)}  & \ding{72}\ding{72} & \ding{72}\ding{73} &  \ding{72}\ding{73} & \ding{72}\ding{73}  \\ 
   \hline

   \makecell{\neucdf (MSCN)}  & \ding{72}\ding{72} & \ding{72}\ding{73} &  \ding{72}\ding{73} & \ding{72}\ding{73}  \\ 
   \hline
\end{tabular}}
\end{table}

\noindent \textbf{\neucdf \textit{v.s.} Direct Selectivity Modeling?}
We summarize the generalization performance of different models w.r.t two popular measures (RMSE and Qerror) in Table~\ref{table.qualitative.compare}. We define three qualitative levels of generalization performance on test sets ---  (\ding{72}\ding{72}): RMSE $< 0.05$ or median Qerror $< 2$; (\ding{72}\ding{73}): $0.05<$ RMSE $< 0.2$ or $2<$ median Qerror $< 10$; (\ding{73}\ding{73}): RMSE $> 0.2$ or median Qerror $> 10$. From the table, we observe two important findings.

\begin{description} [leftmargin=*]  
\item [F1.] All four models achieve very good in-distribution generalization performance w.r.t the metric they are optimized for. Specifically, both LW-NN and MSCN are optimized for Qerror, but after using the \neucdf paradigm, they are optimized for RMSE.
\item [F2.] LW-NN and MSCN perform poorly on OOD queries both in terms of Qerror and RMSE. More importantly, \neucdf can help them achieve much better OOD generalization performance even with Qerror.  This matches the theoretical results regarding OOD generalization error in \S~\ref{section.ood}.
\end{description}

\vspace{-0.7em}
\section{Training with CDF Self-Consistency}
\label{section.training}
Motivated by the theoretical results and the limitation observed in \neucdf, this section introduces a new training framework, \name, for query-driven selectivity models. 


\vspace{-0.5em}
\subsection{\textbf{High-level Idea}\label{section.secon.overview}}

In the previous section, we noted that direct query selectivity modeling is effective for in-distribution generalization with respect to arbitrary measures or loss functions. However, the CDF modeling paradigm used in \neucdf provides superior OOD generalization {\color{black}because it enforces a \emph{hard} constraint on a signed measure, ensuring that all predictions from \neucdf are \emph{coherently} induced by a signed measure.}
However, it does not support arbitrary loss functions, such as Qerror. This raises a key question: \emph{can we combine the advantages of direct query selectivity modeling and CDF modeling to achieve both strong in-distribution generalization with arbitrary loss functions and improved OOD generalization?}


 \name addresses this limitation by adopting direct query selectivity modeling (\textcolor{black}{which avoids the negative estimate issue}) and introducing a \textbf{soft} constraint on the signed measure, unlike the \emph{hard} constraint used in \neucdf.
\textcolor{black}{ Specifically,
\name operates on selectivity learning model $\mathcal{M}$ that targets the query selectivity directly (instead of \neucdf that requires $\mathcal{M}$ to model the CDFs), and applies a \emph{soft} constraint through \emph{CDF self-consistency regularization} during training. Recall that as discussed in \S~\ref{section.neucdf.convert}, \emph{each CDF corresponds to the selectivity estimate of a one-sided rectangle query}, thus we can extract the CDFs learned by the selectivity model~\(\mathcal{M}\) from these queries. We then utilize appropriate loss functions to maintain consistency between the learned CDFs and the learned selectivity function. The intuition is that better \emph{alignment} between the learned selectivity functions and the extracted CDFs indicates that $\mathcal{M}$ is more closely induced by a signed measure, thereby being more likely to achieve bounded OOD generalization error.
}


This approach combines the benefits of both paradigms, providing robust OOD generalization and allowing flexibility in the choice of loss functions.  \textcolor{black}{Although \name is inspired by both the theoretical and empirical analyses of \neucdf, it does not come with a theoretical guarantee because it cannot be confirmed as being entirely induced by signed measures. Despite this, \name shows significant practical effectiveness in our experiments.}

\subsection{\textbf{CDF Self-Consistency Regularization}\label{section.secon.cdf}} \

\begin{figure}
\centering
\includegraphics[height=0.13\textwidth]{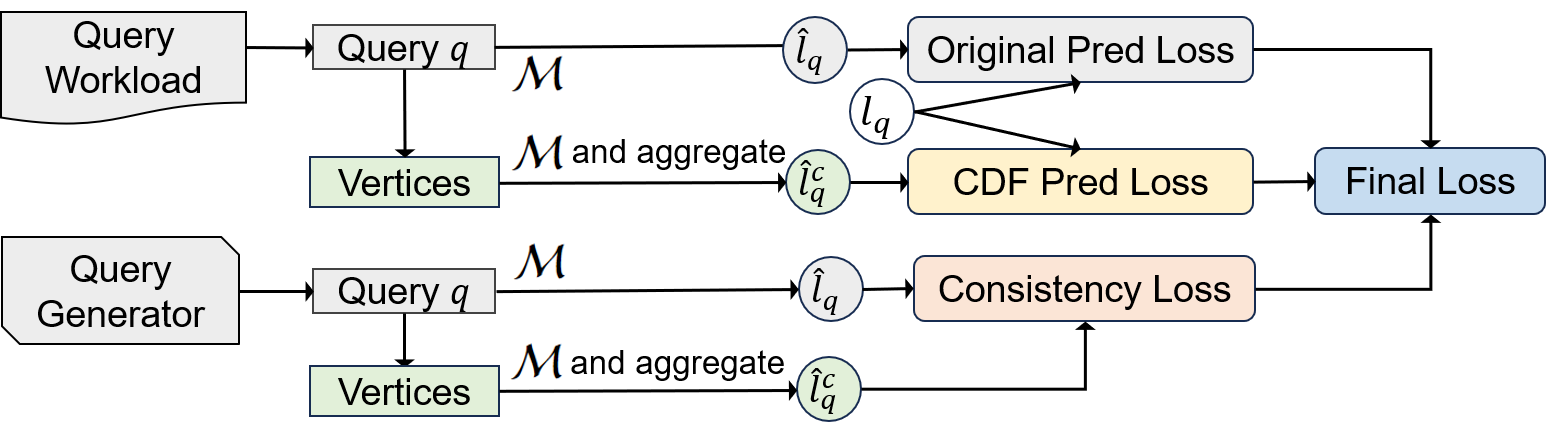}
\vspace{-1em}
\caption{Training a selectivity model $\mathcal{M}$ with \name.}
\label{fig:cdftrain}
\end{figure}

Figure~\ref{fig:cdftrain} illustrates the training workflow of $\mathcal{M}$ using \name. \textcolor{black}{\name processes a \emph{query workload} $\mathcal{W} = \{(q, l)\}$ (same as \neucdf), and utilizes a \emph{query generator} $\mathcal{G}$. They collectively contribute to the final loss optimized by $\mathcal{M}$. The operation of \name within each query batch is described step-by-step.}

\noindent \textcolor{black}{\textbf{\circled{1} Loss Computation with  Query Workload $\mathcal{W}$.} \name initiates with two preprocessing steps analogous to \neucdf: \emph{normalization} and \emph{CDF conversion} (which extracts the set $\{v\}$ of vectors for each query $q$ at its vertices).
With all information needed, \name then computes two types of losses: \textbf{Original Prediction Loss} $\mathcal{L}_{\text{OriPred}}$ and \textbf{CDF Prediction Loss} $\mathcal{L}_{\text{CDFPred}}$.}

\textcolor{black}{
$\mathcal{L}_{\text{OriPred}}$ is calculated as the discrepancy between $\mathcal{M}$'s direct selectivity prediction for a query $q$, denoted $\hat{l}_q$, and the true label $\hat{l}_q$. Typically, the loss function involves Qerror or MSLE, consistent with current methods in query-driven selectivity learning.}

\textcolor{black}{$\mathcal{L}_{\text{CDFPred}}$ aligns with the procedures of \neucdf. For each query $q$, \name transforms its vertex set $\{v\}$ into corresponding one-sided queries and extracts the CDFs as predicted by $\mathcal{M}$. Using the formula~(\ref{eq.cdf}), the selectivity estimate (denoted $\hat{l}^c_q$) from the learned CDFs is calculated, and $\mathcal{L}_{\text{CDFPred}}$ is then defined as the RMSE between $\hat{l}^c_q$ and the actual label ${l}_q$. This loss forces the model $\mathcal{M}$ to learn the underlying CDFs from the \emph{training workload}, aside from the direct mapping from queries to selectivities.
}





\noindent \textcolor{black}{\textbf{\circled{2} Loss Computation with Query Generator $\mathcal{G}$.}  The query generator samples queries from a distribution, using random sampling for this paper, although other sampling methods are compatible within \name. Each sampled query $q$ undergoes the same \emph{normalization} and \emph{CDF conversion} steps as in \circled{1} to produce a set of vertex vectors for $q$. These vectors are used to compute the selectivity estimate $\hat{l}^c_q$ from the learned CDFs. We then introduce a third type of loss, \textbf{Consistency Loss} $\mathcal{L}_{\text{Consistent}}$, defined as the RMSE between $\hat{l}^c_q$ and $\mathcal{M}$'s direct selectivity estimate of $q$, $\hat{l}_q$. This loss enforces \emph{consistency} between $\mathcal{M}$’s direct selectivity predictions and its learned CDFs across diverse queries. This step can be implemented synchronously with \circled{1} to enhance training efficiency.} 
}

\noindent \textcolor{black}{\textbf{\circled{3} Model Training.} The final loss is defined as
\begin{equation}
    \mathcal{L} = \mathcal{L}_{\text{OriPred}} + \omega_1 \mathcal{L}_{\text{CDFPred}} + \omega_2 \mathcal{L}_{\text{Consistent}},
\end{equation}
where $\omega_1, \omega_2$ are hyper-parameters controlling the balance among the three losses. We empirically tune them from four candidate values $\{0.1, 1, 10, 100\}$. $\mathcal{M}$ is optimized to minimize $\mathcal{L}$ using SGD.}

\noindent \textcolor{black}{\textbf{\circled{4} Prediction.} Once trained, $\mathcal{M}$ can directly predict the selectivities for incoming queries without needing CDF conversion.}

\smallskip
\noindent \textbf{Remark.} \name does \emph{not} change the \emph{model architecture} or \emph{inference procedure} of existing selectivity models $\mathcal{M}$ that directly target selectivities. The two losses $\mathcal{L}_{\text{CDFPred}}$ and $\mathcal{L}_{\text{Consistent}}$ serve as the key to incorporating CDF self-consistency regularization into $\mathcal{M}$. Furthermore, computing these losses does not require new actual query executions to obtain selectivities, and it is significantly more efficient than performing queries on a DBMS.

\eat{
\subsection{\textbf{Optimizing Training Efficiency}}

The two regularization losses introduced by \name impose a cost on training efficiency (wall-clock time). This inefficiency stems primarily from the bitmap lookup operations required for both losses and the additional query sampling for consistency loss. Fortunately, the computational overhead associated with CDF prediction loss can be mitigated by pre-loading. Specifically, we pre-compute the CDFs for each query along with their bitmap encodings before training and reuse these computations in subsequent training epochs, as they remain unchanged throughout the training process.

However, the query generator for consistency loss samples new queries on the fly, resulting in queries (including CDFs) and bitmaps that change every epoch. Hence, pre-loading is not applicable here. To address this issue, we employ another strategy --- parallelism. By leveraging distributed training and asynchronous query sampling, we avoid delays caused by the sequential generation of augmented queries (including the associated CDFs) and their bitmap encodings. This approach not only reduces idle time but also enhances overall training throughput.

}

\section{Experiments of \name \label{section.eval}}

In this section, we implement \name and integrate it into two recent NN-based query-driven proposals --- LW-NN~\cite{dutt2019selectivity} and MSCN~\cite{kipf2019learned}. Both models utilize an MLP; however, they adopt distinct methodologies for query encoding. LW-NN employs a flattened query encoding mechanism, while MSCN stands for multi-set convolutional network.  We aim to answer \textcolor{black}{two} research questions as follows.
1) While existing query-driven models perform well for in-distribution generalization, are they robust to OOD generalization? 2) Can \name improve their OOD generalization performance while maintaining their in-distribution performance, in terms of both prediction accuracy and query latency performance?
Note that while we implement \name with two query-driven models, it is general and applicable to any loss-based deep learning models.

\subsection{Experimental Setup\label{section.exp.setup}}
\begin{table*}[!ht]
\centering
\caption{Prediction accuracy on IMDb-small (Left) and DSB (Right) w.r.t.  
 {\textcolor{darkblue}{ in-distribution queries}}/\uline{\textcolor{maroon}{out-of-distribution queries}}.}
\label{table.qerror}
\scalebox{0.76}{
\begin{tabular}{c|c|c|c|c|c|c}
\hline
  \multirow{3}{*}{\textcolor{black}{Model}}  &\multicolumn{3}{c|}{\textcolor{black}{Query Center Move}}  & \multicolumn{3}{c}{\textcolor{black}{Query Granularity Shift}}    \\
  \cline{2-7}
 &  \multirow{2}{*}{\textcolor{black}{RMSE}}  & \multicolumn{2}{c|}{\textcolor{black}{Qerror}}  &   \multirow{2}{*}{\textcolor{black}{RMSE}}  & \multicolumn{2}{c}{\textcolor{black}{Qerror}} 
  \\
\cline{3-4} \cline{6-7}
  & &  \textcolor{black}{Median} & \textcolor{black}{$90\%$}   &     & \textcolor{black}{Median} & \textcolor{black}{ $90\%$}   \\
    \hline
    \textcolor{black}{PostgreSQL}    
    &\textcolor{darkblue}{0.042}/\uline{\textcolor{maroon}{0.086}}  
    & \textcolor{darkblue}{6.3}/\uline{\textcolor{maroon}{4.2}} 
    & \textcolor{darkblue}{669}/\uline{\textcolor{maroon}{549}}
    
    &\textcolor{darkblue}{0.045}/\uline{\textcolor{maroon}{0.124}}  
    & \textcolor{darkblue}{6.1}/\uline{\textcolor{maroon}{3.7}} 
    & \textcolor{darkblue}{921}/\uline{\textcolor{maroon}{297}} \\
    
    \hline
    \textcolor{black}{Sampling} 
    &\textcolor{darkblue}{0.175}/\uline{\textcolor{maroon}{0.196}}  
    & \textcolor{darkblue}{31}/\uline{\textcolor{maroon}{35}} 
    & \textcolor{darkblue}{$10^3$}/\uline{\textcolor{maroon}{$10^3$}}
    
    &\textcolor{darkblue}{0.180}/\uline{\textcolor{maroon}{0.197}}  
    & \textcolor{darkblue}{29}/\uline{\textcolor{maroon}{21}} 
    & \textcolor{darkblue}{$ 10^3$}/\uline{\textcolor{maroon}{$10^3$}} \\
\hline
\textcolor{black}{MSCN}  
& \textcolor{darkblue}{0.020}/\uline{\textcolor{maroon}{0.700}} 
& \textcolor{darkblue}{1.6}/\uline{\textcolor{maroon}{$10^3$}} 
& \textcolor{darkblue}{6.5}/\uline{\textcolor{maroon}{$10^6$}}
&\textcolor{darkblue}{0.021}/\uline{\textcolor{maroon}{0.763}}
& \textcolor{darkblue}{1.5}/\uline{\textcolor{maroon}{$10^3$}}
& \textcolor{darkblue}{8.6}/\uline{\textcolor{maroon}{$10^6$}}
\\
   \hline
    \textcolor{black}{\textbf{MSCN + CDF}}   & \textcolor{darkblue}{0.022}/\uline{\textcolor{maroon}{0.035}} 
& \textcolor{darkblue}{1.9}/\uline{\textcolor{maroon}{2.0}} 
& \textcolor{darkblue}{7.3}/\uline{\textcolor{maroon}{10}}
&\textcolor{darkblue}{0.024}/\uline{\textcolor{maroon}{0.047}}
& \textcolor{darkblue}{1.8}/\uline{\textcolor{maroon}{1.7}}
& \textcolor{darkblue}{11}/\uline{\textcolor{maroon}{7.0}}
 \\ 
   \hline
\end{tabular}
}
\scalebox{0.76}{
\begin{tabular}{c|c|c|c|c|c|c}
\hline
  \multirow{3}{*}{\textcolor{black}{Model}}  &\multicolumn{3}{c|}{\textcolor{black}{Query Center Move}}  & \multicolumn{3}{c}{\textcolor{black}{Query Granularity Shift}}    \\
  \cline{2-7}
 &  \multirow{2}{*}{\textcolor{black}{RMSE}}  & \multicolumn{2}{c|}{\textcolor{black}{Qerror}}  &   \multirow{2}{*}{\textcolor{black}{RMSE}}  & \multicolumn{2}{c}{\textcolor{black}{Qerror}} 
  \\
\cline{3-4} \cline{6-7}
  & &  \textcolor{black}{Median} & \textcolor{black}{$90\%$}   &     & \textcolor{black}{Median} & \textcolor{black}{ $90\%$}   \\
    \hline
    \textcolor{black}{PostgreSQL}    
    &\textcolor{darkblue}{0.033}/\uline{\textcolor{maroon}{0.068}}  
    & \textcolor{darkblue}{1.6}/\uline{\textcolor{maroon}{1.9}} 
    & \textcolor{darkblue}{6.6}/\uline{\textcolor{maroon}{14}}
    
    &\textcolor{darkblue}{0.050}/\uline{\textcolor{maroon}{0.098}}  
    & \textcolor{darkblue}{1.6}/\uline{\textcolor{maroon}{2.8}} 
    & \textcolor{darkblue}{5.2}/\uline{\textcolor{maroon}{15}} \\
    
    \hline
    \textcolor{black}{Sampling} 
    &\textcolor{darkblue}{0.121}/\uline{\textcolor{maroon}{0.186}}  
    & \textcolor{darkblue}{3.7}/\uline{\textcolor{maroon}{9.2}} 
    & \textcolor{darkblue}{38}/\uline{\textcolor{maroon}{82}}
    
    &\textcolor{darkblue}{0.143}/\uline{\textcolor{maroon}{0.194}}  
    & \textcolor{darkblue}{4.7}/\uline{\textcolor{maroon}{12}} 
    & \textcolor{darkblue}{75}/\uline{\textcolor{maroon}{64}} \\
\hline
\textcolor{black}{MSCN}  
& \textcolor{darkblue}{0.057}/\uline{\textcolor{maroon}{0.283}} 
& \textcolor{darkblue}{1.4}/\uline{\textcolor{maroon}{5.0}} 
& \textcolor{darkblue}{3.6}/\uline{\textcolor{maroon}{78}}
&\textcolor{darkblue}{0.027}/\uline{\textcolor{maroon}{0.345}}
& \textcolor{darkblue}{1.2}/\uline{\textcolor{maroon}{$10^3$}}
& \textcolor{darkblue}{1.8}/\uline{\textcolor{maroon}{$10^6$}}
\\
   \hline
    \textcolor{black}{\textbf{MSCN + CDF}}   & \textcolor{darkblue}{0.061}/\uline{\textcolor{maroon}{0.158}} 
& \textcolor{darkblue}{1.6}/\uline{\textcolor{maroon}{2.1}} 
& \textcolor{darkblue}{1.6}/\uline{\textcolor{maroon}{2.1}} 
&\textcolor{darkblue}{5.3}/\uline{\textcolor{maroon}{17}}
& \textcolor{darkblue}{1.3}/\uline{\textcolor{maroon}{2.5}}
& \textcolor{darkblue}{1.8}/\uline{\textcolor{maroon}{73}}
 \\ 
   \hline
\end{tabular}
}
\end{table*}

\noindent \textbf{Datasets.}
We conducted experiments using one single-table dataset, Census, and {\color{black}three} multi-table datasets: IMDb-small, DSB~\cite{ding2021dsb} {\color{black} and CEB~\cite{negi2021flow}}. Census comprises the basic population characteristics in US, with approximately 49K tuples across 13 attributes. We use Census for prediction accuracy experiments since a few relevant approaches only support single-table queries. The IMDb~\cite{leis2015good} dataset is derived from the Internet Movie Database. Previous studies~\cite{leis2015good} show that IMDb is highly correlated and skewed. IMDb-small and CEB use 6 and \textcolor{black}{15 tables} of the original IMDb, respectively.
 DSB is as an extension of the TPC-DS benchmark~\cite{poess2002tpc}, characterized by more complex data distributions and demanding query templates. We populated a DSB database with a scale factor 50 using the default physical design configuration, and use 5 tables in our experiments. 
\smallskip

\noindent \textbf{Workloads.} Since the primary goal of this section is to assess both the in-distribution (In-Dist) and OOD generalization capabilities of query-driven models, we focus on the first two OOD scenarios as outlined in \S~\ref{section.ood.cases}. Specifically, We train models on specific query distributions and assess their performance on unseen queries both within the same distribution (In-Dist generalization) and from different distributions (OOD generalization). To generate such workloads, for each dataset, we initially create a set of candidate queries. For IMDb-small, we directly leverage the training queries from~\cite{kim2022learned} with up to 5 joins and diversified join graphs. For DSB and Census, we create candidate queries by randomly sampling join graphs and filter conditions. {\color{black} However, IMDb-small and DSB are limited to 5- and 4-way star join queries, respectively. To explore the scalability of \name, we extend our analysis to more complex join queries using template 1a of CEB, which includes 9-way joins with star, chain, and self-joins. 
Due to the limited range variation in predicate values, such as the 14 different ranges for \texttt{t.production\_year}, which does not satisfy Assumption~\ref{assump: samp_R}, we have generated new candidate queries from existing CEB-1a queries while enriching the diversity of the \texttt{t.production\_year} ranges. We denote the new 9-way join workload \textbf{CEB-1a-varied}.} After this, we obtained 50K, 60K, 70K {\color{black} and 43K~\footnote{We include all subqueries of CEB-1a-varied queries in the accuracy experiment}} candidate queries for Census, DSB, IMDb-small and \textcolor{black}{CEB-1a-varied}, respectively. From the candidate queries, we simulate training and test workloads for both two OOD scenarios.

To simulate OOD scenarios, we designate a shifting attribute $a$, for each dataset: \texttt{age} for the Census dataset, \texttt{t.production\_year} for IMDb-small \textcolor{black}{and CEB-1a-varied}, and \texttt{ss.ss\_list\_price} for DSB. In both OOD scenarios, models are trained on queries with the attribute $a$ normalized within specific bounds ($c_a$ for query centers and $l_a$ for range lengths). For in-distribution generalization, models are evaluated on queries matching training conditions. For OOD generalization, they are tested on queries where $c_a$ (for center move)  or $l_a$ (for granularity shift) falls outside these bounds. Training and test queries are kept strictly non-overlapping.

 
\smallskip
\noindent \textbf{Compared Approaches.}  We implemented LW-NN~\cite{dutt2019selectivity} ourselves. For MSCN, we used the code from~\cite{mscncode}. We evaluate MSCN and LW-NN trained with \name\footnote{The repository containing the code and data will be included in our official version}, referred to as MSCN+CDF and LW-NN+CDF. We include two query-driven approaches, \pts and \quadh (code from ~\cite{ptscode}), which are based on SOTA theory~\cite{hu2022selectivity}, to demonstrate the limitations of PAC learning. We also include \quick~\cite{park2020quicksel} in our comparison. The three query-driven models are induced from \emph{probability measures} where our OOD generalization result (Thm~\ref{thm: ood}) is applicable. For data-driven approaches, we use \pg (multi-dimensional histograms) and uniform sampling (\samp) as baselines. We do not include other data-driven approaches since this paper focuses on query-driven models. Note that \pts, \quadh, \quick, and the LW-NN we implemented do not support joins, so we evaluate them on Census. For a fair comparison, we exclude data information (\textit{e.g.,} bitmaps) from LW-NN or MSCN, since other query-driven models only utilize query information. We turn on the bitmaps in multi-table experiments. \textcolor{black}{We also compare another strategy for improving generalizability: Robust-MSCN~\cite{negi2023robust} (join bitmaps and query masking), and its variant, Robust-MSCN*, which excludes query masking. Our experiments show that removing query masking improves the performance of Robust-MSCN in the two OOD scenarios (which is likely because we do not include PostgreSQL estimates in the query encoding). We report their results on CEB-1a-varied (which contains the most complex joins, increasing the challenge for query optimization) due to space constraints, as we observe similar trends across other datasets.}

The goal of the experiments is \emph{not} to beat the SOTA query-driven models but to validate the \emph{practicality} of our theory. Specifically, we aim to show that \name, which is designed based on our theory, \emph{reliably} improves upon existing NN-based query-driven models, and \emph{consistently} outperforms the models derived from SOTA theory.

\smallskip

\noindent \textbf{Evaluation Metrics.} For accuracy, we use both RMSE 
and Qerror as the metrics. While our theory assumes absolute error as the loss function, we also evaluate Qerror (which is more critical in query optimization~\cite{moerkotte2009preventing}) to demonstrate the effectiveness of \name. For query latency performance, we report the query running time.

\smallskip

\noindent \textbf{Hardware.} We train all NN models on an Amazon SageMaker ml.g4dn.xlarge node, and conduct latency experiments on an EC2 r5d.2xlarge node (8 core CPUs, 3.1GHz, 64G memory) for IMDb-small and CEB, and on an EC2 c5.9xlarge node (36 core CPUs, 3.1GHz, 72G memory) for DSB.

\subsection{Accuracy~\label{section.exp.accuracy}}

Figure~\ref{figure.exp.census} and Table~\ref{table.qerror} present the prediction accuracy on single-table and multi-table datasets, respectively.
First, deep query-driven models (MSCN and LW-NN) demonstrate superior performance for In-Dist generalization across all datasets and consistently outperform all compared data-driven approaches on multi-table datasets. They perform comparably to \pg on single-table queries, where \pg is already effective.  For \pts and \quadh, despite that they outperform \quick and theoretically benefit from the SOTA theory, they fail to match the empirical performance of the two deep query-driven models due to their limited model capacity, especially for Qerror (Figure~\ref{figure.exp.census.qerror}) since they are optimized specifically for RMSE.
\emph{These findings confirm the In-Dist generalization capability of deep learning-based query-driven models.}

However, they show limited robustness to OOD queries, especially in multi-table datasets with intricate joins and skewed distributions. For example, MSCN achieves strong In-Dist accuracy on \textcolor{black}{the three multi-table datasets}, with median Qerror below 2 and 90th percentile values in single digits. Yet, it struggles with OOD generalization on IMDb-small, where it exhibits an RMSE of about 0.7 and median Qerror in four-digit, significantly underperforming compared to \pg and \samp.
On DSB, MSCN shows less vulnerability to query center shifts. This is likely due to less skewed data distributions, allowing easier adaptation of selectivity functions across different data regions. 
\pts, \quadh, and \quick do not exhibit such drastic drops in OOD performance because they are induced by probability measures. This supports our OOD generalization theory as signed measures are a superset of probability measures and thus fall within the scope of our theory.


\begin{figure}
\centering
\scalebox{0.97}{
\begin{subfigure}{0.235\textwidth}
\captionsetup{justification=centering, skip=0pt}
\includegraphics[height=0.38\textwidth]{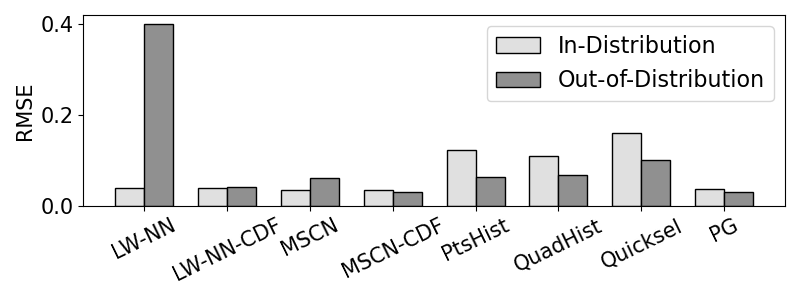}
\caption{RMSE} 
\end{subfigure}
\begin{subfigure}{0.235\textwidth}
\captionsetup{justification=centering, skip=0pt}
\includegraphics[height=0.38\textwidth]{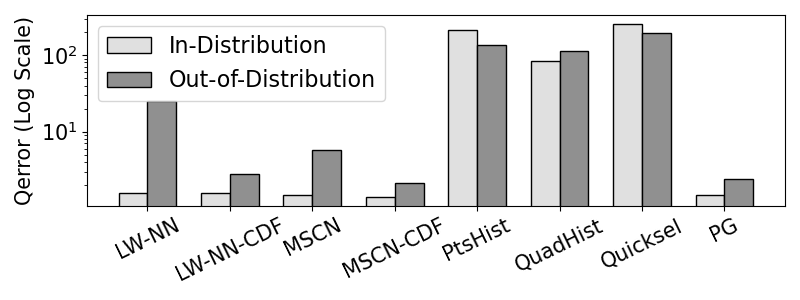}
\caption{Median Qerror~\label{figure.exp.census.qerror}} 
\end{subfigure}
}

\caption{Accuracy on Census with granularity shifts.\label{figure.exp.census}}
\end{figure}

\begin{figure*}
  \caption{\textcolor{black}{Per-query latency performance on CEB-1a-varied under granularity shift. Left: In-Dist queries; Right: OOD queries.}~\label{fig:ceb_qo}}
    \centering
    \begin{subfigure}[b]{0.48\textwidth}
        \includegraphics[width=\textwidth]{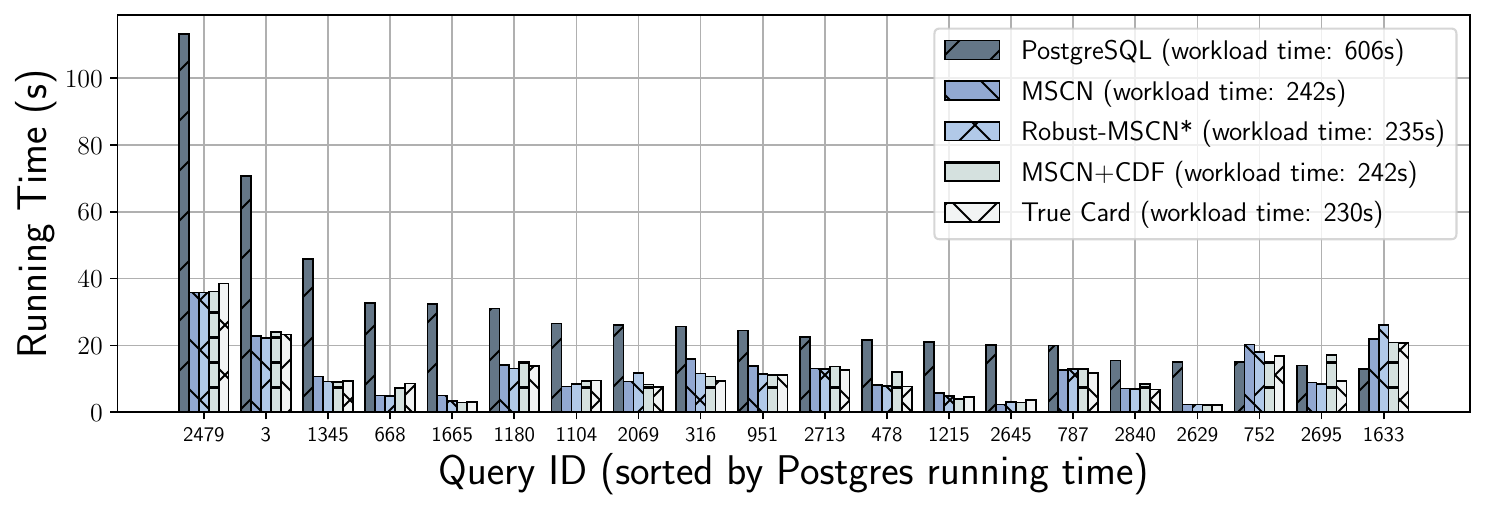}
    \end{subfigure}
    \hfill
    \begin{subfigure}[b]{0.48\textwidth}
        \includegraphics[width=\textwidth]{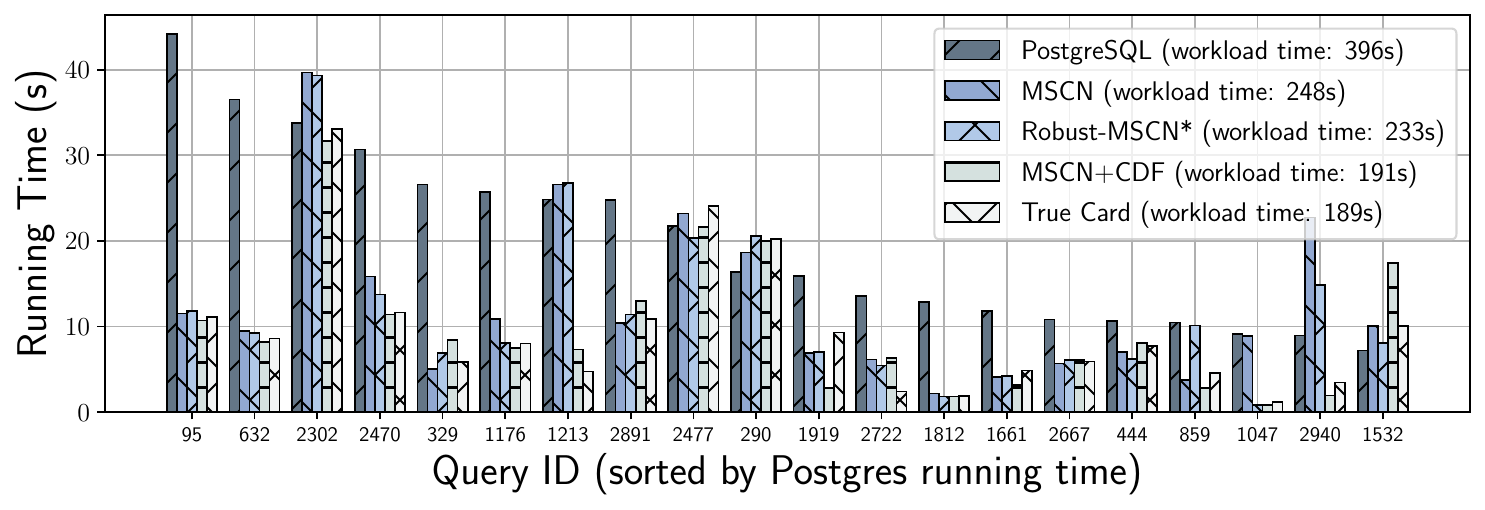}
    \end{subfigure}
\end{figure*}

More importantly and perhaps not surprisingly, the integration of \name significantly enhances the OOD generalization capabilities of query-driven models like MSCN, without compromising their In-Dist generalization. For instance, in the first OOD scenario (query center move), \name training reduces MSCN's median and 90-percentile Qerror from four- and  seven-digit values to just 2 and 10, respectively. Similar dramatic improvements are evident in the second OOD scenario. Moreover, \name does not adversely affect the model's performance on In-Dist generalization. 

\begin{table}
\centering
\scalebox{0.76}{
\centering
\begin{tabular}{c|c|c|c|c|c|c}
\hline
  \multirow{3}{*}{\textcolor{black}{Model}}  &\multicolumn{3}{c|}{\textcolor{black}{Query Center Move}}  & \multicolumn{3}{c}{\textcolor{black}{Query Granularity Shift}}    \\
  \cline{2-7}
 &  \multirow{2}{*}{\textcolor{black}{RMSE}}  & \multicolumn{2}{c|}{\textcolor{black}{Qerror}}  &   \multirow{2}{*}{\textcolor{black}{RMSE}}  & \multicolumn{2}{c}{\textcolor{black}{Qerror}} 
  \\
\cline{3-4} \cline{6-7}
  & &  \textcolor{black}{Median} & \textcolor{black}{$90\%$}   &     & \textcolor{black}{Median} & \textcolor{black}{ $90\%$}   \\
    \hline
    \textcolor{black}{PostgreSQL}    
    &\textcolor{darkblue}{0.054}/\uline{\textcolor{maroon}{0.062}}  
    & \textcolor{darkblue}{22}/\uline{\textcolor{maroon}{6.8}} 
    & \textcolor{darkblue}{$10^3$}/\uline{\textcolor{maroon}{213}}
    
    &\textcolor{darkblue}{0.088}/\uline{\textcolor{maroon}{0.038}}  
    & \textcolor{darkblue}{23}/\uline{\textcolor{maroon}{8.3}} 
    & \textcolor{darkblue}{$10^3$}/\uline{\textcolor{maroon}{269}} \\
    
    \hline
    \textcolor{black}{Sampling} 
    &\textcolor{darkblue}{0.144}/\uline{\textcolor{maroon}{0.089}}  
    & \textcolor{darkblue}{$ 10^3$}/\uline{\textcolor{maroon}{$10^4$}} 
    & \textcolor{darkblue}{$ 10^5$}/\uline{\textcolor{maroon}{$10^5$}}
    
    &\textcolor{darkblue}{0.191}/\uline{\textcolor{maroon}{0.051}}  
    & \textcolor{darkblue}{$ 10^4$}/\uline{\textcolor{maroon}{$10^3$}} 
    & \textcolor{darkblue}{$ 10^5$}/\uline{\textcolor{maroon}{$10^5$}} \\
\hline
\textcolor{black}{MSCN}  
& \textcolor{darkblue}{0.012}/\uline{\textcolor{maroon}{0.045}} 
& \textcolor{darkblue}{1.1}/\uline{\textcolor{maroon}{1.9}} 
& \textcolor{darkblue}{1.5}/\uline{\textcolor{maroon}{42}}
&\textcolor{darkblue}{0.014}/\uline{\textcolor{maroon}{0.117}}
& \textcolor{darkblue}{1.1}/\uline{\textcolor{maroon}{5.8}}
& \textcolor{darkblue}{1.3}/\uline{\textcolor{maroon}{53}}
\\
   \hline
    \textcolor{black}{Robust-MSCN}   &  
    \textcolor{darkblue}{0.019}/\uline{\textcolor{maroon}{0.050}}& \textcolor{darkblue}{1.2}/\uline{\textcolor{maroon}{2.7}} & \textcolor{darkblue}{1.7}/\uline{\textcolor{maroon}{25}}  &
    \textcolor{darkblue}{0.022}/\uline{\textcolor{maroon}{0.152}}& \textcolor{darkblue}{1.2}/\uline{\textcolor{maroon}{7.1}} & \textcolor{darkblue}{1.6}/\uline{\textcolor{maroon}{81}} \\
      \textcolor{black}{Robust-MSCN*}   &  
    \textcolor{darkblue}{0.011}/\uline{\textcolor{maroon}{0.045}}& \textcolor{darkblue}{1.1}/\uline{\textcolor{maroon}{2.1}} & \textcolor{darkblue}{1.3}/\uline{\textcolor{maroon}{15}}  &
    \textcolor{darkblue}{0.015}/\uline{\textcolor{maroon}{0.082}}& \textcolor{darkblue}{1.1}/\uline{\textcolor{maroon}{3.7}} & \textcolor{darkblue}{1.4}/\uline{\textcolor{maroon}{31}} \\
    \hline
    \textcolor{black}{\textbf{MSCN + CDF}}   & \textcolor{darkblue}{0.010}/\uline{\textcolor{maroon}{0.019}} 
& \textcolor{darkblue}{1.1}/\uline{\textcolor{maroon}{1.6}} 
& \textcolor{darkblue}{1.4}/\uline{\textcolor{maroon}{8.7}}
&\textcolor{darkblue}{0.012}/\uline{\textcolor{maroon}{0.012}}
& \textcolor{darkblue}{1.1}/\uline{\textcolor{maroon}{1.5}}
& \textcolor{darkblue}{1.3}/\uline{\textcolor{maroon}{5.4}}
 \\ 
   \hline
\end{tabular}}
\caption{\textcolor{black}{Accuracy on CEB-1a-varied ({\textcolor{darkblue}{In-Dist}}/\uline{\textcolor{maroon}{OOD}}).}~\label{table.ceb.qerror}}
\end{table}
\smallskip
{\color{black} 
\noindent \textbf{More Joins.} Table~\ref{table.ceb.qerror} shows the accuracy over CEB-1a-varied (featuring 9-way joins). We observe similar trends in the previous two multi-table datasets: \name significantly enhances MSCN's OOD performance despite the increased complexity.
}

\smallskip
{\color{black} 
\noindent \textbf{Comparison with Robust-MSCN.} Table~\ref{table.ceb.qerror} reveals that while Robust-MSCN* marginally improves upon MSCN, they are surpassed by \name. This is because Robust-MSCN is tailored for \emph{different OOD scenarios} like new join templates and missing tables/columns, which are not the primary focus of this paper.
}

\smallskip
{\color{black} 
\noindent \textbf{Point Queries.} Figure~\ref{figure.exp.point} presents the OOD performance of MSCN and MSCN+CDF on IMDb-small for point queries. The results indicate that, by treating point queries as range queries, \name still enhances the OOD robustness of MSCN in these cases.

\begin{figure}
  \centering
  \begin{subfigure}[b]{0.235\textwidth}
  \captionsetup{justification=centering, skip=0pt}
      \includegraphics[width=\textwidth]{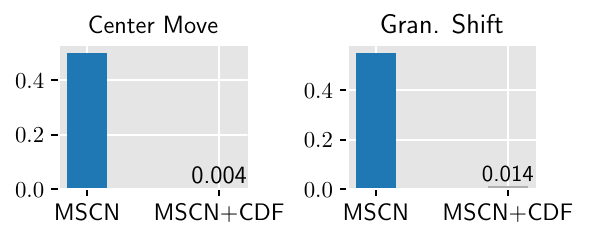}
      \caption{\textcolor{black}{RMSE}}
  \end{subfigure}
  \hfill
  \begin{subfigure}[b]{0.235\textwidth}
  \captionsetup{justification=centering, skip=0pt}
      \includegraphics[width=\textwidth]{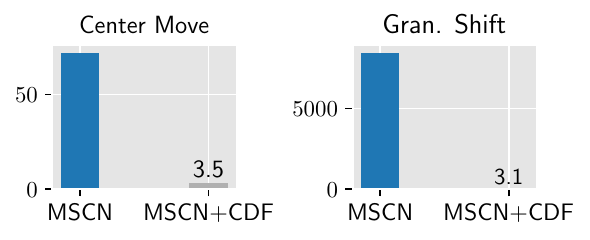}
      \caption{\textcolor{black}{Median Qerror}}
  \end{subfigure}
  
   \caption{\textcolor{black}{Accuracy of OOD point queries on IMDb-small.}~\label{figure.exp.point}}
\end{figure}
}

\subsection{Query Latency Performance~\label{section.exp.e2e}}

In this subsection, we showcase the improved generalization capabilities from \name can result in a better end-to-end performance. All end-to-end experiments are conducted with a modified PostgreSQL 13.1 that can accept injected cardinalities estimates~\cite{pgcode,ceb}. We exclude \samp in the experiments since it is much worse than others. We compare MSCN+CDF to the original MSCN, PostgreSQL (an important baseline upon which learned cardinality estimation should improve), and True cardinalities. For each OOD scenario, we randomly sample 30 queries each from In-Dist and OOD test queries to conduct the latency experiments. The results for OOD queries are shown in Figure~\ref{figure.exp.imdb}. {\color{black} Due to space constraints, we exclude In-Dist performance results,  but we note that both the MSCN and MSCN+CDF demonstrate notably efficient running times for In-Dist queries, significantly surpassing \pg on IMDb-small and matching its performance on DSB. Indeed, they are close to True cardinalities on both datasets.}

\begin{figure}
\centering
\begin{subfigure}{0.235\textwidth}
\captionsetup{justification=centering, skip=3pt}
\includegraphics[height=0.35\textwidth]{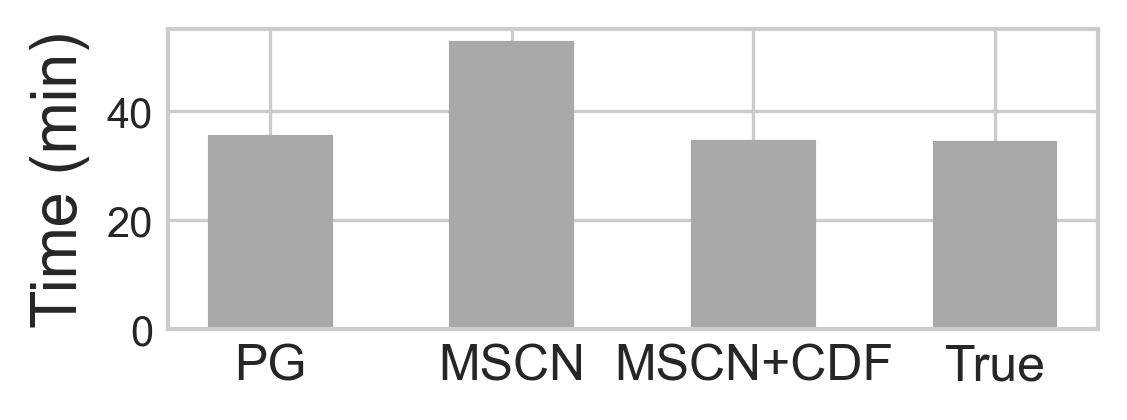}
\vspace{-0.4em}
\caption{center move}       
\end{subfigure}
\begin{subfigure}{0.235\textwidth}
\captionsetup{justification=centering, skip=3pt}
\includegraphics[height=0.35\textwidth]{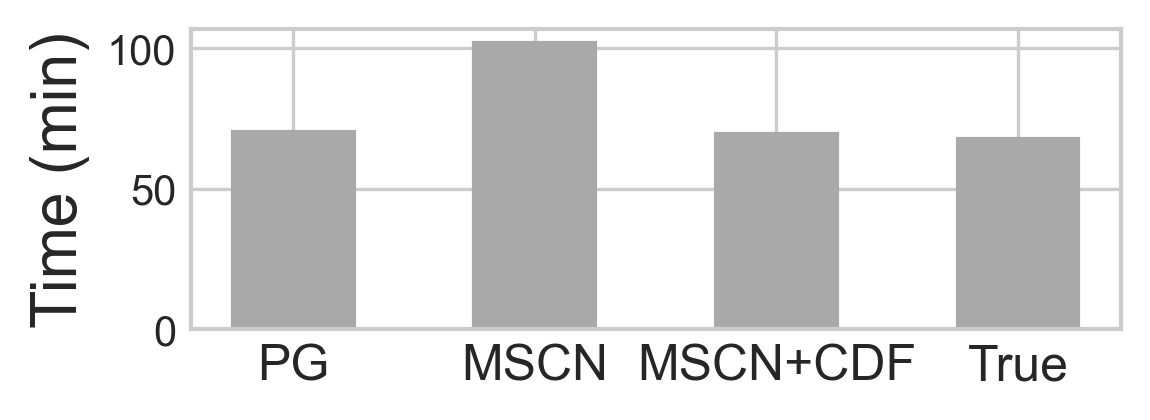}
\vspace{-0.4em}
\caption{\mbox{granularity shift}}     
\end{subfigure}
\begin{subfigure}{0.235\textwidth}
\captionsetup{justification=centering, skip=3pt}
\includegraphics[height=0.35\textwidth]{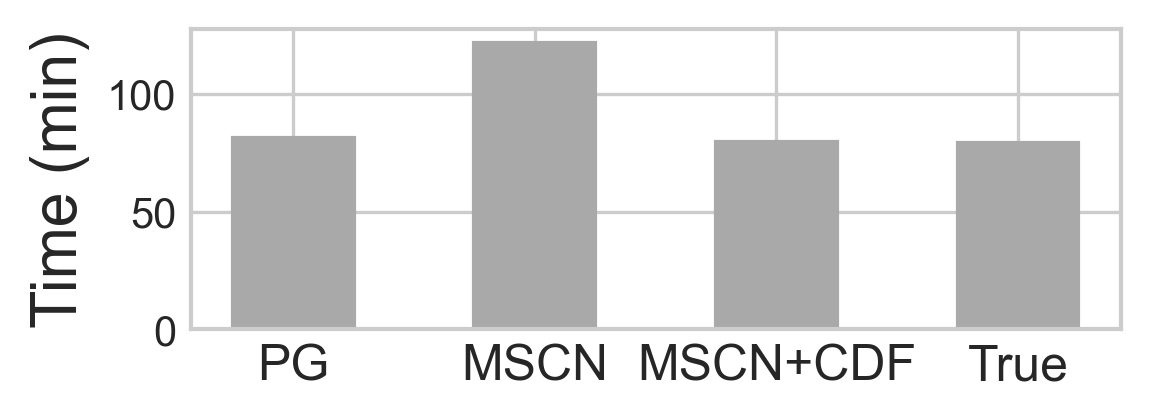}
\vspace{-0.4em}
\caption{center move}     
\end{subfigure}
\begin{subfigure}{0.235\textwidth}
\captionsetup{justification=centering, skip=3pt}
\includegraphics[height=0.35\textwidth]{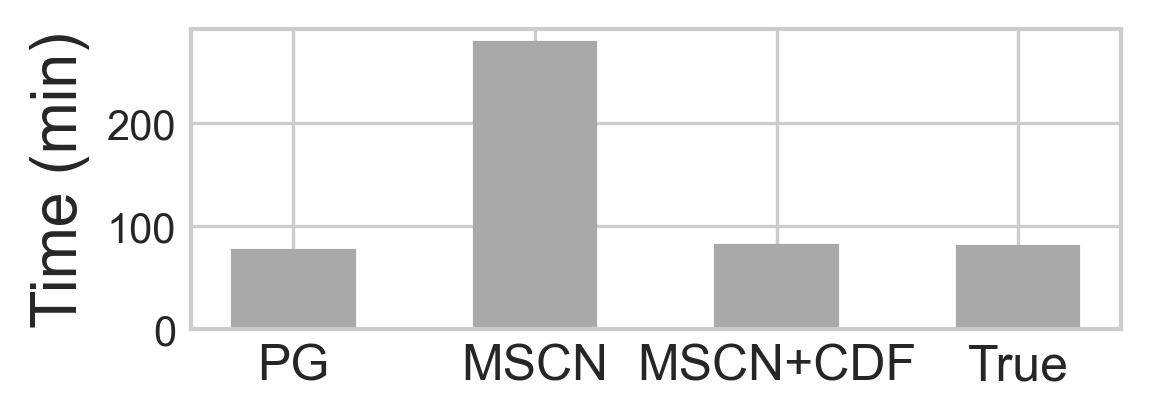}
\vspace{-0.4em}
\caption{\mbox{granularity shift}}      
\end{subfigure}
\caption{OOD query latency performance on IMDb-small (top two subfigures) and DSB (bottom two subfigures).~\label{figure.exp.imdb}}
\end{figure}

The OOD results yield two key insights. First, \emph{the inaccurate cardinality estimates by MSCN for OOD queries lead to considerably poorer query latency performance compared to In-Dist queries.} Notably, MSCN's latency performance is significantly worse than \pg for both IMDb-small OOD queries. Second, the integration of \name significantly enhances MSCN+CDF’s OOD latency performance, bringing it on par with \pg. This demonstrates that the improved accuracy from \name for OOD generalization can translate into enhanced runtime performance.

\smallskip
{\color{black} 
\noindent \textbf{More Joins.} To assess scalability, we conducted latency experiments on CEB-1a-varied, using the same approach to construct workloads of 20 In-Dist and OOD test queries (with 9-way joins) each (we observed consistent results across various sampled workloads). Figure~\ref{fig:ceb_qo} presents per-query latency with workload times indicated in the legend. All MSCN models significantly outperform \pg, as traditional methods struggle with larger numbers of joins. Moreover, MSCN+CDF outperforms MSCN in OOD scenarios. Notably, \name reduces MSCN's running time significantly (by at least a factor of two) in 4 of the 20 OOD queries, with no substantial regressions.  These results confirm that \name scales effectively to more joins.
Additionally, consistent with the observations in \S~\ref{section.exp.accuracy}, \name outperforms Robust-MSCN in enhancing MSCN's query latency for the OOD scenarios discussed in the paper. 
}

\subsection{Efficiency~\label{section.exp.efficiency}}
 \noindent {\color{black} \textbf{Training}.} \name uses pre-loading and asynchronous query sampling {\color{black}(parallelizing two loss computations)} to minimize idle time during training. Training times per epoch are 50s for IMDb-small, 26s for DSB, \textcolor{black}{and $63$s for CEB. MSCN+CDF converges within 80 epochs for all datasets. While the training overheads are higher than MSCN, they are not costly. Additionally, since the training is performed offline, it does not impact real query performance.}

 \noindent {\color{black} \textbf{Inference}. Inference time is crucial for real query performance (including planning and execution). Since \name \emph{does not change the model architecture or inference procedure}, the inference process remains efficient. On a CPU, the average processing time for each query (including subqueries) is 1ms for IMDb-small and DSB, and 14ms for CEB-1a-varied, negligible compared to execution times.

 }



\section{Conclusions and Open Problems~\label{section.conclusion}} 
In this paper, we proved the theory: selectivity predictors induced by a signed measure are learnable, and under mild assumptions, they exhibit \textbf{bounded OOD generalization error}. Based on the theory, we propose a new selectivity estimation paradigm \neucdf, and a principled training framework \name to enhance OOD generalization capabilities for \emph{any} NN-based existing query-driven selectivity models. We empirically demonstrate that \name improves query-driven models' OOD generalization performance in terms of accuracy and query latency performance.

This work opens up many promising research directions. First, extending our theory beyond signed measures could provide new insights. Second, substituting the error function in our theory with Qerror presents an intriguing challenge. Furthermore,  applying our theory to generate queries for effective training is interesting. It is also worth exploring the connection between selectivity learning and the Learning from Label Proportions (LLP) problem~\cite{scott2020learning, zhang2022learning}, as a range query can essentially be viewed as a bag of data points.

\bibliographystyle{acm}
\bibliography{main}
\end{document}